\pdfoutput=1
\documentclass{article}

\usepackage{ijcai22}

\usepackage{times}
\usepackage{soul}
\usepackage[small]{caption}
\usepackage{graphicx}
\usepackage{booktabs}

\usepackage[hidelinks]{hyperref}  
\usepackage[utf8]{inputenc}
\usepackage{url}

\usepackage{amsmath}
\usepackage{amssymb}
\usepackage{bbm}
\usepackage{bm}
\usepackage{cancel}
\usepackage{mathtools}
\usepackage{xfrac}

\usepackage{ascmac}
\usepackage{algorithm}
\usepackage{algpseudocode}

\usepackage{fancyhdr}
\usepackage{tcolorbox}
\usepackage{tikz}

\usepackage{booktabs}
\usepackage{comment}
\usepackage{lastpage}
\usepackage{listings}
\usepackage{multibib}
\usepackage{multirow}
\usepackage[super]{nth}
\usepackage[caption=false]{subfig}

\tcbuselibrary{breakable, skins, theorems}

\allowdisplaybreaks

\newtheorem{theorem}{Theorem}

\newtheorem{proposition}{Proposition}

\newtheorem{lemma}{Lemma}

\newtheorem{definition}{Definition}

\newtheorem{proof}{Proof}

\algtext*{EndFor}
\algtext*{EndWhile}
\algtext*{EndIf}
\algtext*{EndProcedure}
\algtext*{EndFunction}
\algnewcommand{\LineComment}[1]{\State \(\triangleright\) #1}

\newcites{appx}{References}

\newif\ifunderreview
\newif\ifsubmission
\newif\ifappendix

\underreviewfalse

\submissiontrue

\appendixtrue

\ifsubmission
\newcommand{\todo}[1]{}
\newcommand{\replace}[2]{}
\newcommand{\sw}[1]{}
\newcommand{\fh}[1]{}

\else
\newcommand{\todo}[1]{\textbf{\textcolor{red}{[TODO: #1]}}}

\newcommand{\replace}[2]{\textbf{\textcolor{red}{[del: \cancel{#1}]}}\textbf{\textcolor{blue}{[new: #2]}}}

\newcommand{\sw}[1]{\textbf{\textcolor{cyan}{[SW: #1]}}}

\newcommand{\fh}[1]{\textbf{\textcolor{green}{[FH: #1]}}}

\usepackage[inline]{showlabels}

\fi

\makeatletter
\newcommand{\customlabel}[2]{%
   \protected@write \@auxout {}{\string \newlabel {#1}{{#2}{\thepage}{#2}{#1}{}} }%
   \hypertarget{#1}{}
}
\makeatother

\newcommand{\argmin}{\mathop{\mathrm{argmin}}}
\newcommand{\argmax}{\mathop{\mathrm{argmax}}}

\newcommand{\alphav}{\boldsymbol{\alpha}}

\newcommand{\Rgeq}{\mathbb{R}_{\geq 0}}

\newcommand{\fv}{\boldsymbol{f}}
\newcommand{\xv}{\boldsymbol{x}}

\newcommand{\xopt}{\xv_{\text{opt}}}

\newcommand{\prob}{\mathbb{P}}

\newcommand{\B}{\mathcal{B}}
\newcommand{\D}{\mathcal{D}}
\newcommand{\X}{\mathcal{X}}
\newcommand{\Y}{\mathcal{Y}}
\newcommand{\Dv}{\boldsymbol{D}}

\newcommand{\Dl}{\D^{(l)}}
\newcommand{\Dg}{\D^{(g)}}
\newcommand{\Dvl}{\Dv^{(l)}}
\newcommand{\Dvg}{\Dv^{(g)}}

\newcommand{\EI}{\mathrm{EI}}

\newcommand{\Bx}{\B_{\X}}
\newcommand{\Xg}{\X^{\gamma}}

\newcommand{\dtv}{d_{\mathrm{tv}}}

\newcommand{\indic}[1]{\mathbbm{I}[#1]}

\newcommand{\set}[1]{\mathrm{set}(#1)}

\newcommand{\rank}{\stackrel{\mathrm{rank}}{\simeq}}

\newcommand{\citewithname}[2]{#1 \textit{et al.} \shortcite{#2}}

\fancypagestyle{customheader}{
    \fancyhf{} 
    \fancyhead[C]{\raisebox{6ex}[0pt]{\small Proceedings of the Thirty-Second International Joint Conference on Artificial Intelligence (IJCAI-23)}}
}

\begin{document}
\title{
  Speeding Up Multi-Objective Hyperparameter Optimization by \\ Task Similarity-Based Meta-Learning for the Tree-Structured Parzen Estimator
}

\author{
  Shuhei Watanabe \thanks{The work was partially done in AIST.} \textsuperscript{\rm 1}
  \and
  Noor Awad \textsuperscript{\rm 1}
  \and
  Masaki Onishi \textsuperscript{\rm 2}
  \And
  Frank Hutter\textsuperscript{\rm 1} \\
  \affiliations
  \textsuperscript{\rm 1} Department of Computer Science, University of Freiburg, Germany \\
  \textsuperscript{\rm 2} Artificial Intelligence Research Center, AIST, Tokyo, Japan \\
  \emails
  \{watanabs,awad,fh\}@cs.uni-freiburg.de, onishi-masaki@aist.go.jp \\
}

\maketitle

\begin{abstract}
  Hyperparameter optimization (HPO) is a vital step
  in improving performance in deep learning (DL).
  Practitioners are often faced with
  the trade-off between multiple criteria,
  such as accuracy and latency.
  Given the high computational needs of DL and
  the growing demand for efficient HPO,
  the acceleration of multi-objective (MO) optimization
  becomes ever more important.  
  Despite the significant body of work on meta-learning for HPO,
  existing methods are inapplicable to MO tree-structured Parzen estimator (MO-TPE), a simple yet powerful MO-HPO algorithm.
  In this paper, we extend TPE's acquisition function 
  to the meta-learning setting
  using a task similarity defined by the overlap of
  top domains between tasks.
  We also theoretically analyze and address the limitations of our task similarity.
  In the experiments,
  we demonstrate that our method speeds up MO-TPE on
  tabular HPO benchmarks and
  attains state-of-the-art performance.
  Our method was also validated externally by
  winning the \emph{AutoML 2022 competition on ``Multiobjective Hyperparameter Optimization for Transformers''}.
\end{abstract}

\section{Introduction}
\label{main:section:introduction}
Hyperparameter optimization (HPO) is a critical step in achieving
strong performance in deep learning~\cite{chen2018bayesian,henderson-aaai18a}.
Additionally, practitioners are often faced with the trade-off between
important metrics, such as accuracy, latency of inference, memory usage, and
algorithmic fairness~\cite{schmucker2020multi,candelieri2022fair}.
However, exploring the Pareto front of multiple objectives is more complex than single-objective optimization,
making it particularly important to accelerate multi-objective (MO) optimization.

\begin{figure}
  \centering
  \includegraphics[width=0.48\textwidth]{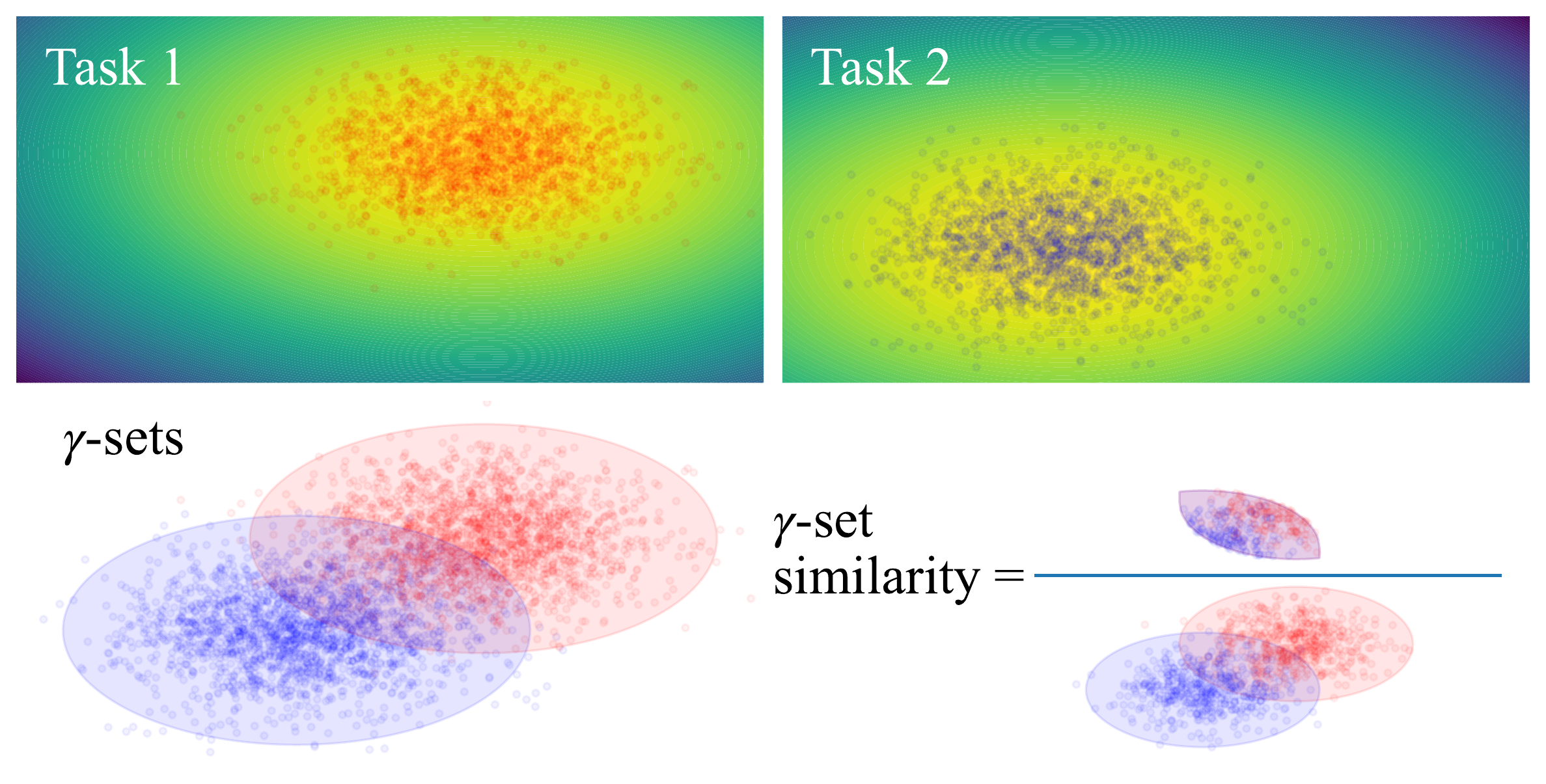}
  \vspace{-6mm}
  \caption{
    The conceptual visualization of
    \emph{$\gamma$-set similarity} measure.
    \textbf{Top row}: the $\gamma$-sets of each task.
    The dots show the top-$\gamma$-quantile observations in both tasks.
    \textbf{Bottom row}: the $\gamma$-set similarity is measured
    via intersection over union
    of the top-$\gamma$-quantile domain,
    which we define \emph{$\gamma$-set};
    see Definition~\ref{main:background:def:gamma-set} in Appendix~\ref{appx:preliminaries:section} for more details.
  }
  \vspace{-4mm}
  \label{main:background:fig:similarity-conceptual}
\end{figure}

To accelerate HPO,
a large body of work on meta-learning has been actively conducted, as surveyed, e.g., by Vanschoren~\shortcite{vanschoren2019meta}.
In the context of HPO, meta-learning mainly focuses on
the knowledge transfer of metadata in Bayesian optimization
(BO)~\cite{swersky2013multi,wistuba2016two,feurer2018practical,perrone2018scalable,salinas2020quantile,Volpp2020Meta-Learning}.
These methods use meta information in Gaussian process (GP) regression to yield
more informed surrogates or an improved acquisition function (AF)
for the target dataset, making them applicable to existing MO-BO methods,
such as ParEGO~\cite{knowles2006parego} and
SMS-EGO~\cite{ponweiser2008multiobjective}.
However, recent works reported that a variant of BO called
MO tree-structured Parzen estimator (MO-TPE)~\cite{ozaki2020multiobjective,ozaki2022multiobjective}
is more effective than the aforementioned GP-based methods in expensive MO settings.
Since MO-TPE uses kernel density estimators (KDEs) instead of GPs,
existing meta-learning methods are not directly applicable,
and a meta-learning procedure for TPE is yet to be explored.

To address this issue,
we propose a meta-learning method for TPE on non-hierarchical spaces,
i.e. search space does not include any conditional parameters,
using a new task kernel.
Our method models the joint probability density function (PDF) of
an HP configuration $\xv$ and a task $t$
using a new task kernel $k_t(t_i, t_j)$.
We calculate the task kernel by using
the intersection-over-union-based new similarity measure,
which we call \emph{$\gamma$-set similarity},
as visualized in
Figure~\ref{main:background:fig:similarity-conceptual}.
Note that we describe the theoretical details in
Appendix~\ref{appx:details-of-meta-learning-tpe:section}.
Although this task kernel
successfully works well in many cases,
its performance is degraded under some circumstances, such as
for high-dimensional spaces or when transferring knowledge from slightly dissimilar tasks.
To alleviate this performance degradation,
we analytically discuss and address the issues in this task kernel
by (1) dimension reduction based on HP importance (HPI)
and (2) an $\epsilon$-greedy algorithm to determine the next HP configuration.

In our experiments,
we demonstrate that our method successfully speeds up MO-TPE
(or at least recovers the performance of MO-TPE when meta-tasks are not similar).
The effectiveness of our method was also validated externally by winning the
\emph{AutoML 2022 competition on ``Multiobjective Hyperparameter Optimization for Transformers''}.
Note that this paper serves as the public announcement of the winner solution as well.

In summary, the main contributions of this paper are to:
\begin{enumerate}
  \item extend TPE acquisition function (AF) to the meta-learning setting
        using a new task kernel,
        \vspace{-1mm}
  \item discuss the drawbacks of the task kernel
        and provide the solutions to them, and
        \vspace{-1mm}
  \item validate the performance of our method on real-world tabular benchmarks next to the external competition.
        \vspace{-2mm}
\end{enumerate}
To facilitate reproducibility, our source code is available at \url{https://github.com/nabenabe0928/meta-learn-tpe}.


\section{Related Work}
\label{main:related-work:section}
In the context of BO,
MO optimization is handled either by
reducing MO to a single-objective problem (scalarization)
or employing an AF that measures
utility of a new configuration in the objective space.
ParEGO~\cite{knowles2006parego}
is an example of scalarization that enjoys
a convergence guarantee to the Pareto front.
SMS-EGO~\cite{ponweiser2008multiobjective} uses a lower-confidence bound
of each objective to calculate hypervolume (HV) improvement,
and EHVI~\cite{emmerich2011hypervolume}
uses expected HV improvement.
PESMO~\cite{hernandez2016predictive} and MESMO~\cite{wang2017max}
are the extensions for MO settings of predictive entropy search and
max-value entropy search.
While those methods rely on GP, MO-TPE
uses KDE and was shown to outperform the aforementioned methods in expensive MO-HPO
settings~\cite{ozaki2020multiobjective,ozaki2022multiobjective}.

The evolutionary algorithm (EA) community also studies MO actively. MOEAs
use either surrogate-assisted EAs (SAEA)~\cite{chugh2016surrogate,guo2018heterogeneous,pan2018classification}
or non-SAEA methods.
Non-SAEA methods, such as
NSGA-II~\cite{deb2002fast}
and MOEA/D~\cite{zhang2007moea},
typically require thousands of evaluations to converge~\cite{ozaki2020multiobjective},
and thus SAEAs are currently more dominant in the EA domain.
Since SAEAs combine an EA with a cheap-to-evaluate surrogate,
SAEAs are essentially similar to BO, as they can be seen as using EAs to optimize a particular AF in BO.

Meta-learning~\cite{vanschoren2019meta} is a popular method to accelerate optimization and most of them can be classified into either of the following five types in the context of HPO:
\begin{enumerate}
  \vspace{-1mm}
  \item initialization (or warm-starting) using promising configurations in meta-tasks~\cite{feurer2015initializing,nomura2021warm},
  \vspace{-1mm}
  \item search space reduction~\cite{wistuba2015hyperparameter,perrone2019learning},
  \vspace{-1mm}
  \item learning an AF~\cite{Volpp2020Meta-Learning},
  \vspace{-1mm}
  \item linear combination of models trained on each task~\cite{wistuba2016two,feurer2018practical}, and
  \vspace{-1mm}
  \item training of a model jointly with meta-tasks~\cite{swersky2013multi,springenberg2016bayesian,perrone2018scalable,salinas2020quantile}.
  \vspace{-1mm}
\end{enumerate}
Warm-starting helps especially at the early stage of optimizations but does not use knowledge from the metadata afterward.
Search space reduction could be applied to any method, but cannot identify the best configurations if the target task's optimum is outside of the optima for the meta-train tasks.
The learning of AFs applies an expensive reinforcement learning step and is specific to GP-based methods.
The linear combination is empirically demonstrated to outperform
most meta-learning BO methods~\cite{feurer2018practical} including the search space reduction.
The joint model trains a model on both observations and metadata.
Although the linear combination of models (Type 4) is simple yet empirically strong, no meta-learning scheme for TPE has been developed so far.
For this reason, we introduce a meta-learning method via a joint model (Type 5) inspired by Type 4.


\section{Background}
\label{main:background:section}

\subsection{Bayesian Optimization (BO)}
Suppose we would like to \textbf{minimize} a loss
metric $f(\xv)$,
then HPO can be formalized as follows:
\begin{equation}
\begin{aligned}
  \xopt \in \argmin_{\xv \in \X} f(\xv)
\end{aligned}
\end{equation}
where $\X \coloneqq \X_1 \times \dots \X_D \subseteq \mathbb{R}^D$
is the search space and
$\X_d \subseteq \mathbb{R}$ ($d = 1,\dots,D$) is the domain of the $d$-th HP.
In Bayesian optimization (BO)~\cite{brochu-arXiv10a,shahriari-ieee16a,garnett-bayesian22},
we assume that $f(\xv)$
is expensive, and we consider the optimization
in a surrogate space given observations $\D$.
First, we build a predictive model $p(f|\xv, \D)$.
Then, the optimization in each iteration is replaced
with the optimization of the so-called AF.
A common choice for the AF is
the following expected improvement~\cite{jones1998efficient}:
\begin{equation}
  \begin{aligned}
    \mathrm{EI}_{f^\star}[\xv | \D] = 
    \int_{-\infty}^{f^\star}
    (f^\star - f)p(f | \xv, \D) df.
  \end{aligned}
  \label{expected-improvement-eq}
\end{equation}
Another common choice is the following probability of improvement (PI)~\cite{kushner1964new}:
\begin{equation}
  \begin{aligned}
    \prob[f \leq f^\star | \xv, \D] =
    \int_{-\infty}^{f^\star} p(f | \xv, \D) df.
  \end{aligned}
\end{equation}

\subsection{Tree-Structured Parzen Estimator (TPE)}
TPE~\cite{bergstra2011algorithms,bergstra2013making}
is a variant of BO methods and it uses the expected improvement.
See Watanabe~\shortcite{watanabe2023tpe} to better understand the algorithm components.
To transform Eq.~(\ref{expected-improvement-eq}),
we define the following:
\begin{eqnarray}
  p(\xv | f, \D) \coloneqq \left\{
  \begin{array}{ll}
     p(\xv | \Dl) & (f \leq f^\gamma) \\
     p(\xv | \Dg) & (f > f^\gamma)
  \end{array}
  \right.
  \label{eq:l-and-g-trick}
\end{eqnarray}
where $\Dl, \Dg$ are the observations 
with $f(\xv_n) \leq f^\gamma$ and $f(\xv_n) > f^\gamma$, respectively.
$f^\gamma$ is determined such that $f^\gamma$ is
the $\lceil \gamma |\D| \rceil$-th best observation in $\D$.
Note that $p(\xv | \Dl), p(\xv | \Dg)$
are built by KDE.
Combining Eqs.~(\ref{expected-improvement-eq}), (\ref{eq:l-and-g-trick})
and Bayes' theorem,
the AF of TPE is computed as~\cite{bergstra2011algorithms}:
\begin{equation}
\begin{aligned}
  \mathrm{EI}_{f^\gamma}[\xv | \D]
  &\rank \frac{p(\xv | \Dl)}{p(\xv | \Dg)}.
\end{aligned}
\label{tpe-transformation}
\end{equation}
Note that $\phi(\xv) \rank \psi(\xv)$ implies
the order isomorphic and 
$\forall \xv, \xv^\prime \in \X, \phi(\xv) \leq \phi(\xv^\prime) \Leftrightarrow \psi(\xv) \leq \psi(\xv^\prime)$ holds.
In each iteration, TPE samples configurations from
$p(\xv | \Dl)$ and takes the configuration
that satisfies the maximum density ratio
among the samples.
Note that although our task kernel cannot be computed for tree-structured search space,
a.k.a. non-hierarchical search space, we use the name tree-structured Parzen estimator because this name is already recognized as a BO method using the density ratio of KDEs.

\subsection{Multi-Objective TPE (MO-TPE)}
MO-TPE~\cite{ozaki2020multiobjective,ozaki2022multiobjective}
is a generalization of TPE
with MO settings which falls back to the original TPE
in case of single-objective settings.
MO-TPE also uses the density ratio $p(\xv | \Dl) / p(\xv|\Dg)$
and picks the configuration with the best AF value
at each iteration.
The only difference from the original TPE is the split algorithm
of $\D$ into $\Dl$ and $\Dg$.
MO-TPE uses the HV subset selection
problem (HSSP)~\cite{bader2011hype} to obtain $\Dl$.
HSSP tie-breaks configurations with the same non-domination rank based on the HV contribution.
MO-TPE is reduced to the original TPE when we apply it to a single objective problem.
In this paper, we replace HSSP with a simple tie-breaking method
based on the crowding distance~\cite{deb2002fast}
as this method does not require HV calculation,
which can be highly expensive.

\section{Meta-Learning for TPE}
\label{main:meta-learning-tpe:section}
In this section, we briefly explain
the TPE formulation and then describe
the formulation of the AF for the meta-learning setting.
Note that our method can be easily extended
to MO settings using a rank metric
$R: \mathbb{R}^m \rightarrow \mathbb{R}$ of
an objective vector $\fv \in \mathbb{R}^m$,
and thus we discuss our formulation for the single-objective setting
for simplicity;
see Appendix~\ref{appx:generalization-tpe:section} for the theoretical discussion of the extension to MO settings.

Throughout this paper, we denote metadata
as $\Dv \coloneqq \{\D_m\}_{m=1}^T$,
where $T \in \mathbb{N}$ is the number of tasks and $\D_m$ is the set of observations on the $m$-th task
with size $N_m \coloneqq |\D_m|$.
We use the notion of the $\gamma$-set,
which is, roughly speaking, a set of top-$\gamma$ quantile configurations
as visualized in Figure~\ref{main:background:fig:similarity-conceptual};
for more theoretical details, see Appendix~\ref{appx:preliminaries:section}.
Furthermore, we define $\Xg_m$ as the $\gamma$-set of the $m$-th task.
For example, the red regions and the blue regions in Figure~\ref{main:background:fig:similarity-conceptual} correspond to $\Xg_1$ and $\Xg_2$.

\subsection{Task-Conditioned Acquisition Function}
TPE~\cite{bergstra2011algorithms} first splits a set of observations
$\D = \{(\xv_n, f(\xv_n))\}_{n=1}^N$ into $\Dl$ and $\Dg$
at the top-$\gamma$ quantile.
Then we build KDEs $p(\xv|\Dl)$ and $p(\xv|\Dg)$, and compute the AF via $p(\xv|\Dl)/p(\xv|\Dg)$.
The following proposition provides the multi-task version of the AF:
\begin{proposition}
  Under the assumption of the conditional shift,
  the task-conditioned $\mathrm{AF}$ is computed as$\mathrm{:}$
  \begin{equation}
    \begin{aligned}
      \mathrm{EI}_{f^\gamma}[\xv | t, \Dv]
      \rank \frac{p(\xv, t | \Dvl)}{p(\xv, t | \Dvg)}.
    \end{aligned}
    \label{main:methods:eq:task-conditioned-acq-fn}
  \end{equation}
\end{proposition}
The conditional shift means that $p(\xv | y, t_i) = p(\xv | y, t_j)$ holds for different tasks, i.e. $t_i \neq t_j$ and it holds in our formulation due to the classification nature of the TPE model.
We discuss more details in Appendix~\ref{appx:details-of-acq-fn:section}.
This formulation transfers the knowledge of top domains
and weights the knowledge from similar tasks more.
To compute the AF, we need to model the joint PDFs
$p(\xv, t|\Dl), p(\xv, t|\Dg)$,
which we thus discuss in the next section.

\subsection{Task Kernel}
To compute the task kernel $k_t(t_i, t_j)$,
the $\gamma$-set similarity visualized in Figure~\ref{main:background:fig:similarity-conceptual} (see Appendix~\ref{appx:preliminaries:section} for the formal definition) is employed.
From Theorem~\ref{main:methods:theorem:convergence-of-gamma-similarity}
in Appendix~\ref{appx:details-of-task-similarity:section},
\begin{equation}
\begin{aligned}
  \hat{s}(\Dl_i, \Dl_j) \coloneqq \frac{
    1 - \dtv(p_i, p_j)
  }{
    1 + \dtv(p_i, p_j)
  }
\end{aligned}
\end{equation}
almost surely converges
to the $\gamma$-set similarity $s(\Xg_i, \Xg_j)$
if we can guarantee the strong consistency of $p(\xv|\Dl_m)$
for all $m = 1,\dots,T$
where we define $p_m \coloneqq p(\xv | \Dl_m)$,
$t_m$ as a meta-task
for $m = 2, \dots, T$
and $t_1$ as the target task,
\begin{equation}
\begin{aligned}
  \dtv(p_i, p_j) \coloneqq 
  \frac{1}{2}\int_{\xv \in \X} | p(\xv|\Dl_i) - p(\xv|\Dl_j)| d\xv 
\end{aligned}
\label{main:methods:eq:total-variation}
\end{equation}
is the total variation distance,
and $p(\xv|\Dl_i)$ is estimated by KDE.
Note that $\dtv(p_i, p_j)$
is approximated simply via Monte-Carlo sampling.
In short, we need to compute:
\begin{enumerate}
  \item KDEs of the top-$\gamma$-quantile observations in $\D_m$, and
  \item $\dtv$ between the target task and each meta-task.
\end{enumerate}
Then we define the task kernel as follows:
\begin{eqnarray}
  k_t(t_i, t_j) = \left\{
  \begin{array}{ll}
    \frac{1}{T}
     \hat{s}(\Dl_{i}, \Dl_{j})
     & (i \neq j) \\
     1 - \frac{1}{T}\sum_{k \neq i}
     \hat{s}(\Dl_i, \Dl_k)
     & (i = j)
  \end{array}
  \right..
  \label{main:methods:eq:task-similarity}
\end{eqnarray}
Note that our task kernel is not strictly a kernel function as our task kernel does not satisfy semi-positive definite although it is still symmetric.
The kernel is defined so that the summation over all tasks is $1$, and then KDEs are built as follows:
\begin{equation}
\begin{aligned}
  p(\xv, t | \Dv^\prime) &=
  \frac{1}{N_{\mathrm{all}}^\prime}
  \sum_{m=1}^T k_t(t, t_m)
  \sum_{n=1}^{N_m^\prime} k_x(\xv, \xv_{m,n}) \\
  &= \frac{1}{N_{\mathrm{all}}^\prime}\sum_{m=1}^T
  N_m^\prime k_t(t, t_m)
  p(\xv | \D_m^\prime),
\end{aligned}
\end{equation}
where $\Dv^\prime \coloneqq \{\D_m^\prime\}_{m=1}^T$
is a set of subsets of the observations on the $m$-th
task $\D^\prime_m = \{(\xv_{m,n}, f_m(\xv_{m,n}))\}_{n=i}^{i + N^\prime_m - 1}$, and
$N_{\mathrm{all}}^\prime = \sum_{m=1}^T N_m^\prime$.
In principle, $\D^\prime_m$ could be either $\Dl_m$ or $\Dg_m$.
The advantages of this formulation are to
(1) not be affected by the information from another task $t_m$
if the task is \emph{dissimilar} from the target task $t_1$,
i.e. $\hat{s}(t_1, t_m) = 0$,
and
(2) asymptotically converge to the original formulation as
the sample size goes to infinity,
i.e. $\lim_{N_1^\prime \rightarrow \infty} p(\xv,t|\Dv^\prime) = p(\xv | \D^\prime_1)$.


\begin{algorithm}[tb]
  \caption{Task kernel (after the modifications)}
  \label{main:methods:alg:task-similarity}
  \begin{algorithmic}[1]
    \Statex{$\eta$ (controls the dimension reduction amount), $S$ (Sample size of Monte-Carlo sampling)}
    \Statex{$l_m(\xv) \coloneqq p(\xv|\Dl_m)$ for $m = 1, \dots, T$}
    \For{$d = 1, \dots, D$}\Comment{Dimension reduction}
    \State{Calculate the average HPI $\bar{\mathbb{V}}_d$
      based on Eq.~(\ref{appx:high-dim:eq:anova})
    }
    \EndFor
    \LineComment{Pick dimensions from higher $\bar{\mathbb{V}}_d$}
    \State{Build $\mathcal{S}$ with the top-$\lfloor \log_\eta |\Dl_1| \rfloor$ dimensions}
    \State{Re-build $p^{\mathrm{DR}}_m(\xv|\Dl_m)$ based on Eq.~(\ref{main:high-dim:eq:reduced-kde})}
    \For{$m= 2, \dots, T$}
    \LineComment{Use $S$ samples for Monte-Carlo sampling}
    \State{Calculate $\dtv$ in Eq.~(\ref{main:methods:eq:total-variation})
      with $l^{\mathrm{DR}}_1, l^{\mathrm{DR}}_m$
    }
    \State{Calculate $k_t(t_1, t_m)$ based on Eq.~(\ref{main:methods:eq:task-similarity})}
    \EndFor
    \State{\textbf{return} $k_t$}
  \end{algorithmic}
\end{algorithm}

\begin{algorithm}[tb]
  \caption{Meta-learning TPE}
  \label{main:methods:alg:meta-tpe}
  \begin{algorithmic}[1]
    \State{$N_{\mathrm{init}}$ (the number of initial samples),
      $N_s$ (the number of candidates for each iteration),
      $\gamma$ (the quantile to split $\D_{\cdot}$),
      $\epsilon$ (the ratio of random sampling), $\D_m$ (metadata)}
    \State{$\D_1 \leftarrow \emptyset, \D_{\mathrm{init}} \leftarrow \emptyset$}
    \For{$m = 2, \dots, T$} \Comment{\textcolor{cyan}{Create a warm-start set}}
    \label{main:methods:line:start-of-warm-start}
    \State{Add the top $\lceil N_{\mathrm{init}} / (T - 1) \rceil$ in $\D_m$ to $\D_{\mathrm{init}}$}
    \LineComment{Build KDEs for meta-tasks}
    \State{Sort $\D_m$ and build KDEs $p(\xv|\Dl_m), p(\xv|\Dg_m)$}
    \EndFor
    \For{$n = 1, \dots, N_{\mathrm{init}}$} \Comment{\textcolor{cyan}{Initialization by warm-start}}
    \State{Randomly pick $\xv$ from $\D_{\mathrm{init}}$}
    \State{Pop $\xv$ from $\D_{\mathrm{init}}$}
    \State{$\D_1 \leftarrow \D_1 \cup \{(\xv, f_1(\xv))\}$}
    \EndFor
    \label{main:methods:line:end-of-warm-start}
    \While{Budget is left}
    \State{$\mathcal{S} = \emptyset$}
    \State{Sort $\D_1$ and build KDEs $p(\xv|\Dl_1), p(\xv|\Dg_1)$}
    \For{$m = 1, \dots, T$}
    \State{$\{\xv_j\}_{j=1}^{N_s} \sim p(\xv | \Dl_m), \mathcal{S} \leftarrow \mathcal{S} \cup \{\xv_j\}_{j=1}^{N_s}$}
    \EndFor
    \State{\textcolor{cyan}{
        Calculate the task kernel $k_t$ by Algorithm~\ref{main:methods:alg:task-similarity}
      }}
    \If{$r \leq \epsilon$}\Comment{$r \sim \mathcal{U}(0, 1)$, \textcolor{cyan}{$\epsilon$-greedy algorithm}}
    \label{main:methods:line:eps-greedy}
    \State{Randomly sample $\xv$ and set $\xopt \leftarrow \xv$}
    \Else
    \State{Pick $\xopt \in \argmax_{\xv \in \mathcal{S}} \mathrm{EI}_{f^\gamma}[\xv|t_1,\Dv]$}
    \Comment{Eq.~(\ref{main:methods:eq:task-conditioned-acq-fn})}
    \EndIf
    \State{$\D_1 \leftarrow \D_1 \cup \{(\xopt, f_1(\xopt))\}$}
    \EndWhile
  \end{algorithmic}
  \label{main:methods:alg:meta-learn-tpe}
\end{algorithm}

\subsection{When Does Our Meta-Learning Fail?}
\label{main:when-does-method-fail:section}
In this section,
we discuss the drawbacks of our meta-learning method
and provide solutions for them.

\subsubsection{Case I: $\gamma$-Set for Target Task $\Dl_1$ Does Not Approach $\Xg$}
From the assumption of Theorem~\ref{main:methods:theorem:convergence-of-gamma-similarity},
$\Dl_1$ must approach $\Xg$ to approximate the $\gamma$-set similarity precisely;
however, since TPE is not a uniform sampler
and it is a local search method
due to the fact that the AF of TPE is PI~\cite{watanabe2022ctpe,watanabe2023ctpe,song2022general},
it does not guarantee that $\Dl_1$ goes to $\Xg$
and it may even be guided towards non-$\gamma$-set domains.
In this case, our task similarity measure not only obtains
a wrong approximation, but also causes poor solutions.
To avoid this problem, we introduce the $\epsilon$-greedy
algorithm to pick the next configuration instead of the greedy algorithm.
By introducing the $\epsilon$-greedy algorithm,
we obtain the following theorem,
and thus we can guarantee more correct or tighter similarity approximation:
\begin{theorem}
  If we use the $\epsilon$-greedy policy ($\epsilon \in (0, 1)$) for $\mathrm{TPE}$ to choose
  the next candidate $\xv$ for a noise-free objective function $f(\xv)$
  defined on search space $\X$ with at most a countable number of configurations,
  and we use a $\mathrm{KDE}$ whose distribution converges to the empirical distribution as
  the number of samples goes to infinity,
  then a set of the top-$\gamma$-quantile observations $\Dl_1$ is almost surely
  a subset of $\Xg$.
  \label{main:methods:theorem:gamma-set-converges}
\end{theorem}
The proof is provided in Appendix~\ref{appendix:proofs:subsection:proof-of-gamma-set-convergence}.
Intuitively speaking, this theorem states that when we use the $\epsilon$-greedy algorithm,
$\Dl_1$ will not include any configurations worse than the top-$\gamma$ quantile 
if we have a sufficiently large number of observations.
Therefore, the task similarity is correctly or pessimistically estimated.
Notice that we use the bandwidth selection used by \citewithname{Falkner}{falkner2018bohb},
which satisfies the assumption about the KDE in Theorem~\ref{main:methods:theorem:gamma-set-converges}.

\begin{figure*}
  \centering
  \includegraphics[width=0.92\textwidth]{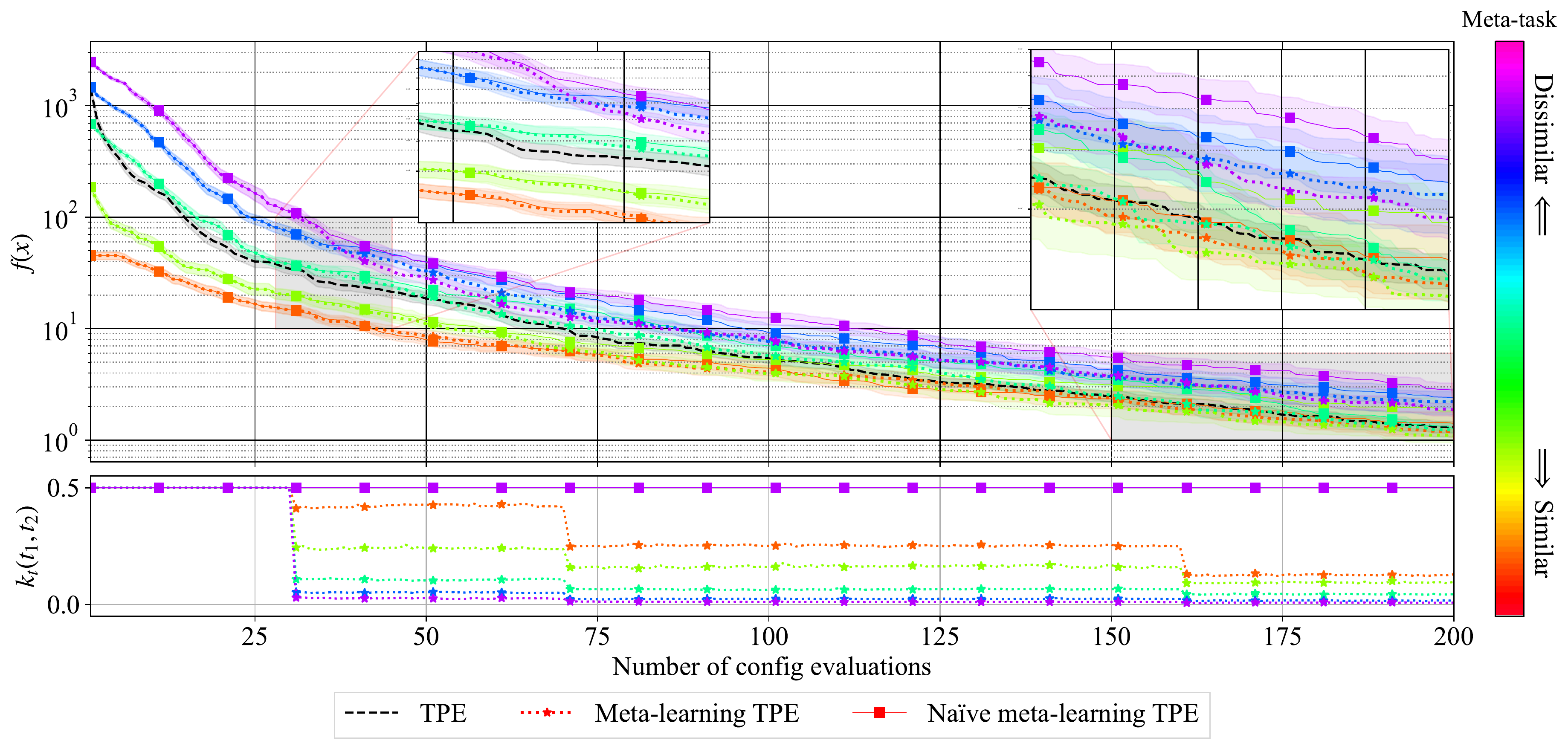}
  \vspace{-2mm}
  \caption{
    The comparison of the convergence of TPE and meta-learning TPE
    based on the task similarity.
    $c^\star = 0$ is identical to the target task
    and tasks become dissimilar as $c^\star$ becomes larger.
    \textbf{Top}: each line except the black line is the performance curves
    of meta-learning TPE on differently similar tasks (orange is similar and purple is dissimilar).
    Dotted lines with $\star$ markers are for
    meta-learning TPE
    and solid lines with $\blacksquare$ markers are for
    na\"ive meta-learning TPE.
    Weak-color bands show the standard error of the objective function value
    over $50$ independent runs.
    \textbf{Bottom}:
    the medians of the task weight on the meta-task (higher is similar).
  }
  \vspace{-2mm}
  \label{appx:case-I:fig:similarity-vs-convergence}
\end{figure*}

\begin{figure}[t]
  \centering
  \includegraphics[width=0.48\textwidth]{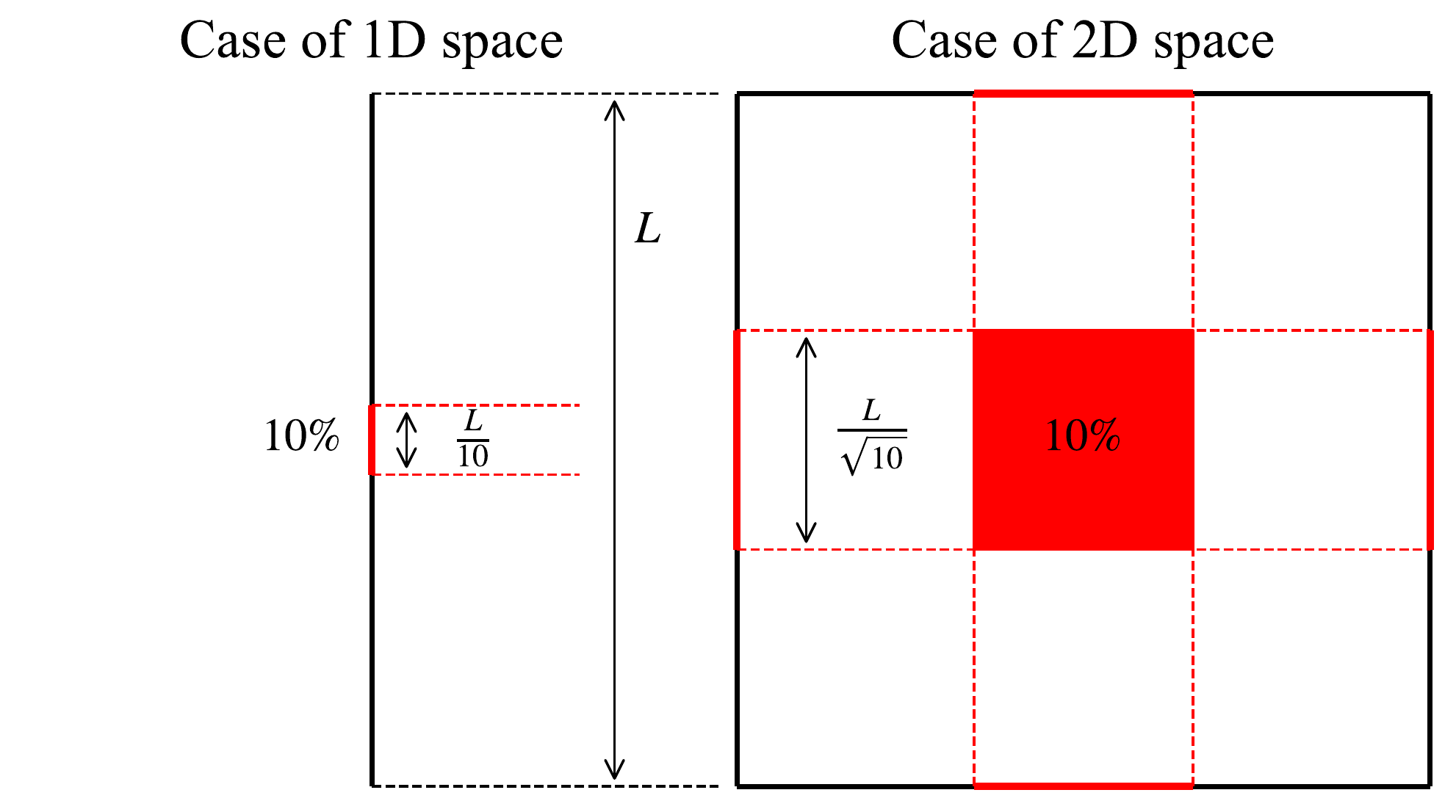}
  \caption{
    The conceptual visualization where the top-$10\%$ domain
    for each dimension becomes larger when the importance
    of each dimension is same and there is no interaction between dimensions.
    The thick black lines are the edges of each domain
    and the red lines are the important domains for each dimension.
    The red lines become longer as the dimensionality becomes higher and it implies that the marginal $\gamma$-set PDF approaches the uniform PDF as $D$ goes to $\infty$.
  }
  \label{appx:case-II:fig:high-dim-conceptual}
\end{figure}

\subsubsection{Case II: Search Space Dimension $D$ Is High}
When the dimensionality $D$ is high, the approximation of the $\gamma$-set similarity
is easily biased.
In Figure~\ref{appx:case-II:fig:high-dim-conceptual},
we provide a concrete example,
where we consider $f(\xv) = \|R\xv\|_1$.
Note that $\xv \in [-1/2, 1/2]^D$
and $R \in \mathbb{R}^{D\times D}$ is the rotation matrix in this example.
The $\gamma$-set of this example is $\Xg = [-\gamma^{1/D}/2, \gamma^{1/D}/2]$.
As $\lim_{D\rightarrow \infty}\gamma^{1/D} = 1$,
which can be seen from the fact that the red lines become longer in 2D case,
the marginal $\gamma$-set PDF $p_d(x_d|\Dl) \coloneqq \int_{\xv_{-d} \in \X_{-d}} p(\xv|\Dl) d\xv_{-d}$ for each dimension (roughly speaking, the red lines for each dimension in Figure~\ref{appx:case-II:fig:high-dim-conceptual} show the region where the marginal $\gamma$-set PDF exhibits high density) approaches the uniform PDF in this example.
However, the marginal $\gamma$-set PDF for each dimension
will not converge to the uniform PDF when
we have only a few observations.
In fact, it is empirically known that
the effective dimension
$D_e$ in HPO is typically much lower than $D$~\cite{bergstra2012random}.
This could imply that the marginal $\gamma$-set PDF for most dimensions
go to the uniform PDF
and only a fraction of dimensions have non-uniform PDF.
In such cases, the following holds:
\begin{proposition}
  If some dimensions are trivial on two tasks,
  the $\gamma$-set similarity between those tasks
  is identical irrespective of
  with or without those dimensions.
  \label{main:methods:proposition:trivial-dims-not-matter}
\end{proposition}
Roughly speaking, a trivial dimension is a dimension that does not contribute to $f(\xv)$ at all.
The formal definition of the trivial dimensions and the proof are provided in Appendix~\ref{appendix:proofs:subsection:proof-of-trivial-dims-not-matter}.
As HP selection does not change the $\gamma$-set similarity under such circumstances,
we would like to employ a dimension reduction method
and we choose an HPI-based dimension reduction.
The reason behind this choice is that 
HPO often has categorical parameters
and other methods, such as principle component analysis and  
singular value decomposition,
mix up categorical and numerical parameters.
In this paper, we use PED-ANOVA \cite{watanabe2023ped},
which computes HPI for each dimension via Pearson divergence between the marginal $\gamma$-set PDF
and the uniform PDF.

Algorithm~\ref{main:methods:alg:task-similarity}
includes the pseudocode of the dimension reduction.
We first compute HPIs in Eq.~(16) by \citewithname{Watanabe}{watanabe2023ped} for each dimension and
take the average of HPI:
\begin{equation}
\begin{aligned}
  \mathbb{V}_{d,m} &\coloneqq \gamma^2
  \mathbb{E}_{x_d \sim \X_d}\biggl[
    \biggl(
      \frac{p_d(x_d | \Dl_m)}{u(\X_d)} - 1
    \biggr)^2
  \biggr], \\
  \bar{\mathbb{V}}_d &\coloneqq
  \frac{1}{T}\sum_{m=1}^T \mathbb{V}_{d,m} \\
\end{aligned}
\label{appx:high-dim:eq:anova}
\end{equation}
where $u(\X_d)$ is the uniform PDF defined on $\X_d$.
Then we pick the top-$\lfloor \log_\eta |\Dl_1| \rfloor$ dimensions
with respect to $\bar{\mathbb{V}}_d$
and define the set of dimensions $\mathcal{I} \in 2^{\{1,\dots,D\}}$.
While we compute the original KDE via:
\begin{equation}
\begin{aligned}
  p(\xv | \D^\prime) = \frac{1}{N}\sum_{n=1}^N
  \prod_{d=1}^D k_d(x_d, x_{d,n})
\end{aligned}
\end{equation}
where
$\D^\prime \coloneqq \{(x_{1,n}, x_{2,n},\dots, x_{D,n}, f(\xv_n))\}_{n=1}^N$
and $k_d$ is the kernel function for the $d$-th dimension,
we compute the reduced PDF via:
\begin{equation}
\begin{aligned}
  p^{\mathrm{DR}}(\xv | \D^\prime) = \frac{1}{N}\sum_{n=1}^N
  \prod_{d \in \mathcal{I}} k_d(x_d, x_{d,n}).
  \label{main:high-dim:eq:reduced-kde}
\end{aligned}
\end{equation}

\subsection{Algorithm Description}
Algorithm~\ref{main:methods:alg:meta-tpe}
presents the whole pseudocode of our meta-learning TPE
and the color-coding shows our propositions.
To stabilize the approximation of the task kernel,
we employ the dimension reduction shown in
Algorithm~\ref{main:methods:alg:task-similarity}
and the $\epsilon$-greedy algorithm at the optimization
of the AF in Line~\ref{main:methods:line:eps-greedy} of
Algorithm~\ref{main:methods:alg:meta-tpe}
as discussed in Section~\ref{main:when-does-method-fail:section}.
Furthermore,
we use the warm-start initialization as seen in
Lines~\ref{main:methods:line:start-of-warm-start} -- \ref{main:methods:line:end-of-warm-start}
of Algorithm~\ref{main:methods:alg:meta-tpe}.
The warm-start further speeds up optimizations.
Note that we apply the same warm-start initialization to all meta-learning methods
for fair comparisons in the experiments.
To extend our method to MO settings,
all we need is to employ a rank metric
$R: \mathbb{R}^m \rightarrow \mathbb{R}$ for an objective vector $\fv$
as mentioned in Appendix~\ref{appx:generalization-tpe:section}.
In this paper, we consistently use the non-domination rank
and the crowding distance for all methods to realize fair comparisons.

\begin{figure*}[t]
  \centering
  \includegraphics[width=0.98\textwidth]{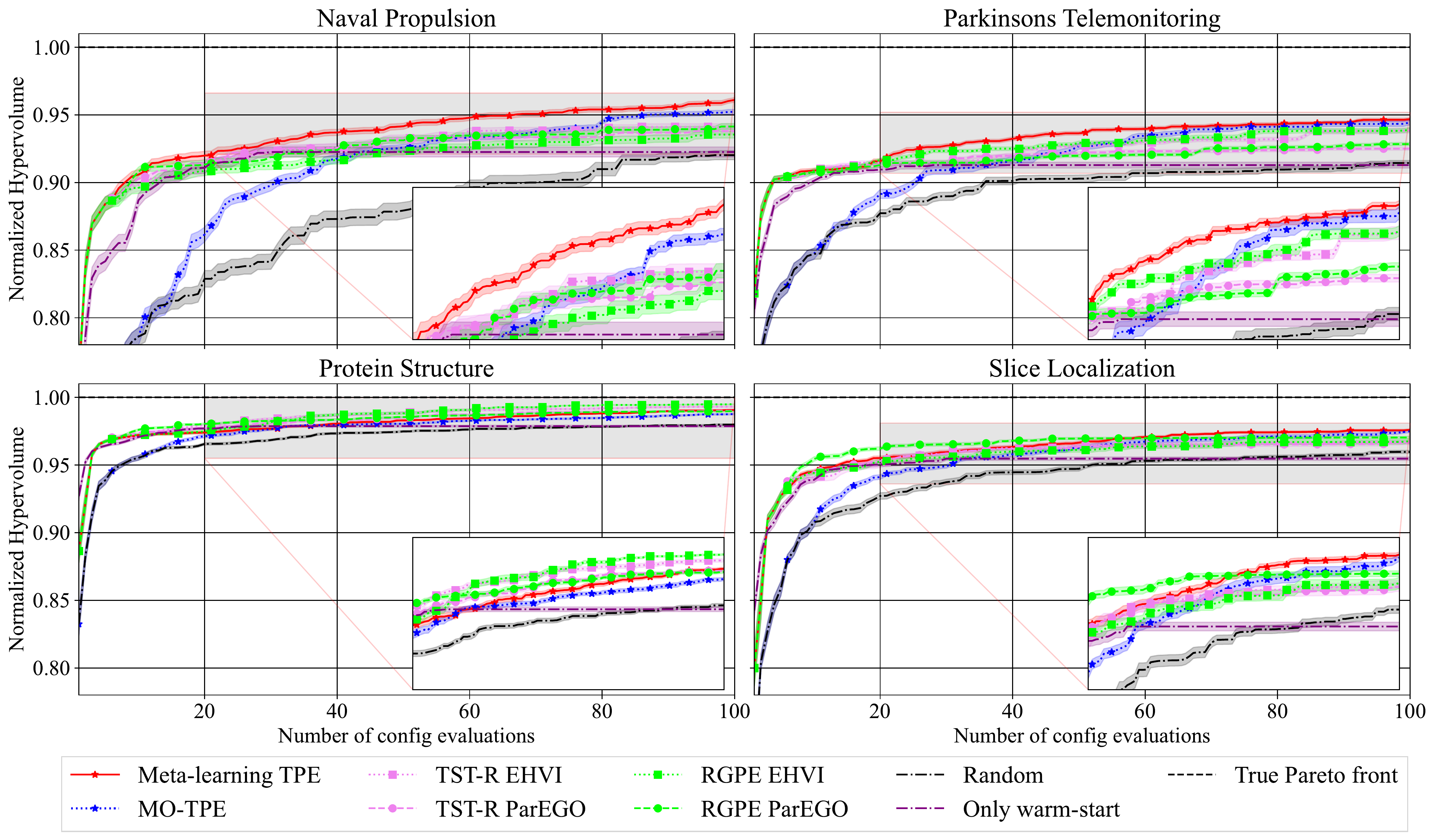}
  \caption{
    The normalized HV over time on four
    joint neural architecture search and hyperparameter optimization benchmarks (HPOlib)
    from HPOBench.
    Each method was run with $20$ different random seeds
    and the weak-color bands present standard error.
    The small inset figures in each figure
    are the magnified gray areas.
    See Appendix~\ref{appx:additional-results:section} for the Pareto fronts achieved by $50\%$~\protect\cite{watanabe2023pareto} of the runs.
  }
  \label{main:experiments:fig:hv-hpolib}
\end{figure*}

\subsection{Validation of Modifications}

To see the effect of the task kernel,
we conduct an experiment using the ellipsoid function
$f(\xv | c) \coloneqq f(x_1, \dots, x_4 | c) = \sum_{d = 1}^4 5^{d - 1} (x_d - c)^2$
defined on $[-5, 5]^4$.
Along with the original TPE, we optimized $f(\xv | c = 0)$
by meta-learning TPE using the randomly sampled $100$ observations from
$f(\xv | c_\star)$ where $c^\star \in [0, 1, \dots, 4]$
(each run uses only one of $[0, 1, \dots, 4]$) as metadata.
Furthermore, we also evaluated na\"ive meta-learning TPE
that considers $k_t(t_i, t_j) = 1/T$ for all pairs of tasks.
All control parameter settings followed Section~\ref{main:experiments:section}
except we evaluated $200$ configurations.

In Figure~\ref{appx:case-I:fig:similarity-vs-convergence},
we present the result.
The top figure shows the performance curve
and the bottom figure shows the weight ($k_t(t_1, t_2) \in [0, 1]$)
on the meta-task.
As seen in the figure, the performance rank is proportional to
the task similarity in the early stage of the optimizations,
and thus meta-learning TPE on the dissimilar tasks performed poorly in the beginning.
However, as the number of evaluations increases,
the performance curves of dissimilar meta-tasks
quickly approach that of TPE after $30$ evaluations
where $\lfloor \log_{\eta}|\Dl_1| \rfloor$ first becomes non-zero
for $\eta = 2.5$.
We can also see the task weight is also ordered
by the similarity between the target task and the meta-task.
Thanks to this effect, on the dissimilar tasks (purple, blue lines), our method starts to recover
the performance of TPE from that point (see the first inset figure)
and our method showed closer performance to the original TPE compared to the na\"ive meta-learning TPE
(see the second inset figure).
This result demonstrates the robustness of our method
to the knowledge transfer from dissimilar meta-tasks.
For similar tasks (light green, orange lines), our method accelerates at the early stage
and slowly converges to the performance of TPE.
Notice that since we use random search for the metadata and the observations are from the meta-learning TPE sampler, which is obviously a non-i.i.d sampler due to the iterative update nature, the top-$\gamma$-quantile in observations obtained from our method is concentrated in a subset of the true top-$\gamma$-quantile as stated in Theorem~\ref{main:methods:theorem:gamma-set-converges}.
It implies that the weight for meta-task is expected to
decrease over time even if the target task is identical to meta-tasks.

\begin{figure*}[t]
  \centering
  \includegraphics[width=0.98\textwidth]{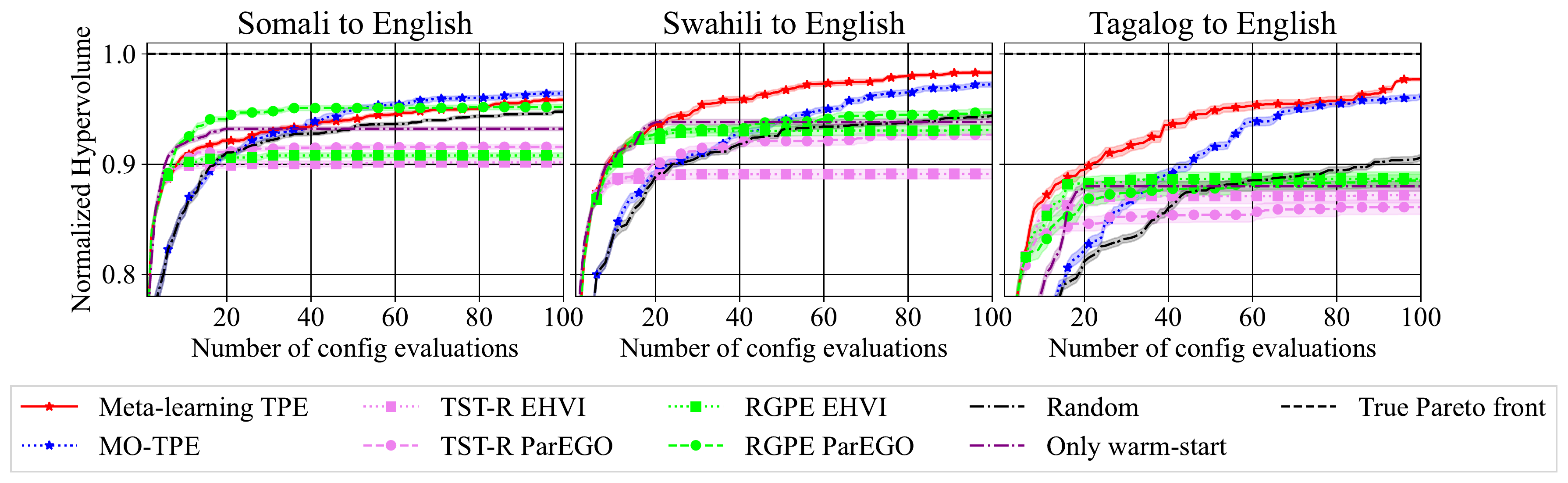}
  \caption{
    The normalized HV over time on NMT-Bench.
    Each method was run with $20$ different random seeds
    and the weak-color bands present standard error.
    See Appendix~\ref{appx:additional-results:section} for the Pareto fronts achieved by $50\%$~\protect\cite{watanabe2023pareto} of the runs.
  }
  \label{main:experiments:fig:hv-nmt}
\end{figure*}

\section{Experiments}
\label{main:experiments:section}

\subsection{Setup}
In the experiments, we optimize
two metrics (a validation loss or accuracy metric, and runtime)
on four joint NAS \& HPO benchmarks (HPOlib)~\cite{klein2019tabular} in HPOBench~\cite{eggensperger-neuripsdbt21},
as well as NMT-Bench~\cite{zhang2020reproducible}.
The baselines are as follows:
\begin{enumerate}
  \vspace{-1mm}
  \item RGPE either with EHVI or ParEGO,
  \vspace{-1mm}
  \item TST-R either with EHVI or ParEGO,
  \vspace{-1mm}
  \item Random search,
  \vspace{-1mm}
  \item Only warm-start (top-$10\%$ configurations in metadata),
  \vspace{-1mm}
  \item MO-TPE~\cite{ozaki2020multiobjective,ozaki2022multiobjective}.
  \vspace{-1mm}
\end{enumerate}
RGPE~\cite{feurer2018practical} and TST-R~\cite{wistuba2016two} were reported to show the best average performance in
a diverse set of meta-learning BO methods~\cite{feurer2018practical}.
Note that since RGPE and TST-R require a rank metric
to compute the ranking loss,
we used non-domination rank and crowding distance,
which we use for meta-learning TPE as well.
Each meta-learning method 
uses $100$ random configurations from the other datasets in each tabular benchmark
and uses the warm-start
initialization ($N_{\mathrm{init}} = 5$) in Algorithm~\ref{main:methods:alg:meta-learn-tpe}.
\texttt{Only warm-start} serves as an indicator of how good the initialization could be and we can judge whether warm-start helps or the meta-learning methods help. 
In this setup, Algorithm~\ref{main:methods:alg:task-similarity} takes $1.0 \times 10^{-4}$ seconds
for a $10D$ space with $200$ observations for $5$ meta-tasks.
In Appendix~\ref{appx:details-of-experiments:section}, we describe more details about control parameters of each method
and tabular benchmarks and discuss the effect of the control parameter $\eta$ and the number of meta-tasks on the performance.

The performance of each experiment was measured via the normalized HV.
When we define the worst (maximum) and the best (minimum) values of each objective as
$f^{\max}_i, f^{\min}_i$ for $i \in \{1,\dots,M\}$,
the normalized HV is computed as:
\begin{equation}
\begin{aligned}
  \prod_{i=1}^M\frac{ f^{\max}_i - f_i}{f^{\max}_i - f^{\min}_i}.
\end{aligned}
\end{equation}
In principle, the normalized HV is better when it is higher
and the possible best value is 1, which is shown as \emph{True Pareto front}.
Although the normalized HV curve tells
us how much each method could improve solution sets,
it does not visualize how solutions distribute
in the objective space (on average).
Therefore, we also provide
the $50\%$ empirical attainment surfaces~\cite{fonseca1996performance}
in Appendix~\ref{appx:additional-results:section}.

\subsection{Results on Real-World Tabular Benchmarks}
Figure~\ref{main:experiments:fig:hv-hpolib}
shows the results on HPOlib.
For all datasets, \texttt{Only warm-start} quickly yields
results slightly worse than Random search with as many as $100$ function evaluations.
This implies that the knowledge transfer surely helps
at the early stage of each optimization,
but each method still needs to explore better configurations.
Our meta-learning TPE method shows the best performance curves
except for \texttt{Protein Structure};
however, our method did not exhibit
the best performance on \texttt{Protein Structure},
where its performance is still competitive.
Furthermore, while we found out that \texttt{Protein Structure}
is relatively dissimilar from the other benchmarks (as discussed in Appendix~\ref{appx:additional-results:section}),
our method could still outperform the non-transfer MO-TPE.

Figure~\ref{main:experiments:fig:hv-nmt}
shows the results on NMT-Bench.
As in HPOlib, \texttt{Only warm-start} quickly yields
results slightly worse than Random search with as many as $100$ function evaluations
for all datasets in NMT-Bench as well.
On the other hand, while RGPE and TST-R
exhibit performance indistinguishable from \texttt{Only warm-start}
in most cases,
our method still improved until the end.
This implies that \texttt{Only warm-start} is not sufficient in this task
and our method could make use of the knowledge from metadata.
However, our meta-learning TPE method struggled in \texttt{Somali to English},
which was dissimilar to the other datasets as discussed
in Appendix~\ref{appx:additional-results:section}.
Although our method did not exhibit the best performance in this case,
it still recovered well with enough observations and got the best performance in the end.

\section{Conclusion}

In this paper, we introduced a multi-task training method
for TPE and demonstrated the performance of our method.
Our method measures the similarity between tasks
using the intersection over union
and computes the task kernel based on the similarity.
As the vanilla version of this simple meta-learning method has some drawbacks,
we employed the $\epsilon$-greedy algorithm and the dimension reduction
method to stabilize our task kernel.
In the experiment using a synthetic function,
we confirmed that our task kernel correctly
ranks the task similarity
and our method could demonstrate
robust performance over various task similarities.
In the real-world experiments,
we used two tabular benchmarks of expensive HPO problems.
Our method could outperform other methods
in most settings and still exhibit competitive performance in the worst case of dissimilar tasks.
Our method's strong performance is also backed by winning the \emph{AutoML 2022 competition on ``Multiobjective Hyperparameter Optimization for Transformers''}.

\clearpage
\clearpage

\section*{Acknowledgments}
The authors appreciate the valuable contributions of the anonymous reviewers and helpful feedback from Ryu Minegishi.
Robert Bosch GmbH is acknowledged for financial support.
The authors also acknowledge funding by European Research Council (ERC) Consolidator Grant ``Deep Learning 2.0'' (grant no.\ 101045765).
Views and opinions expressed are however those of the authors only and do not necessarily reflect those of the European Union or the ERC.
Neither the European Union nor the ERC can be held responsible for them.

\begin{center}
  \includegraphics[width=0.3\textwidth]{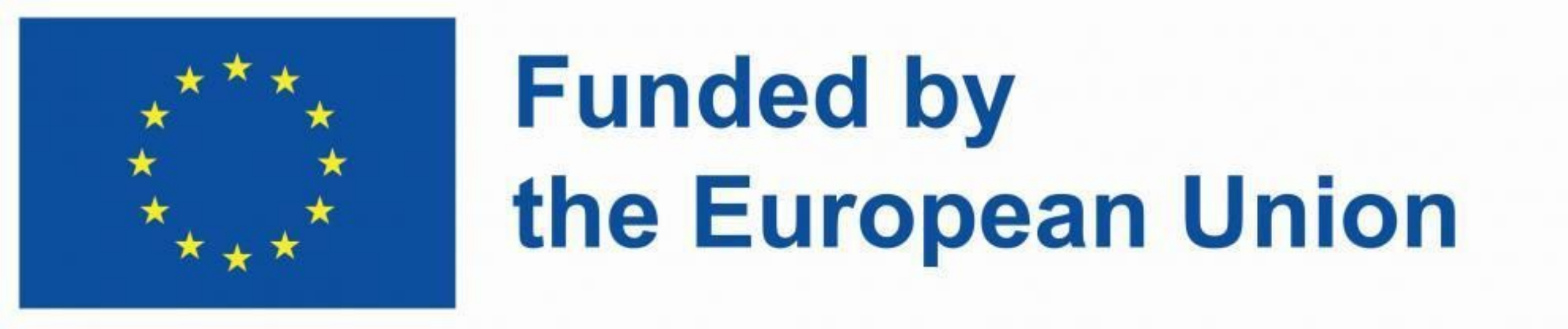}
\end{center}

\bibliographystyle{bib-style}
\bibliography{ref}

\ifappendix
\clearpage
\appendix
\section{Details of Meta-Learning TPE}
\label{appx:details-of-meta-learning-tpe:section}

In Appendix, $\mu$ is the Lebesgue measure defined on the search space $\X$ and $\B_D$ is the Borel body over $\mathbb{R}^D$ if not specified.

\subsection{Task-Conditioned Acquisition Function}
\label{appx:details-of-acq-fn:section}
To begin with, we explain the transformation of TPE:
\begin{equation}
\begin{aligned}
  \EI_{f^\gamma}[\xv |\D] &=
  \int_{-\infty}^{f^\gamma} (f^\gamma - f)
  p(f|\xv) df \\
  &= \int_{-\infty}^{f^\gamma} (f^\gamma - f)
  \frac{p(\xv|f, \D)p(f|\D)}{p(\xv|\D)}df \\
  & ~~~~~(\because \mathrm{Bayes'~theorem}) \\
  &= \frac{p(\xv|\Dl)}{p(\xv|\D)}
  \underbrace{\int_{-\infty}^{f^\gamma} (f^\gamma - f) p(f|\D)df}_{
    \mathrm{const~w.r.t.~}\xv
  }
  ~(\because \mathrm{Eq.}~(\ref{eq:l-and-g-trick})) \\
  &\propto \frac{p(\xv|\Dl)}{\gamma p(\xv|\Dl) + (1 - \gamma)p(\xv|\Dg) }\\
  &~\biggl(\because
  p(\xv|\D) = \int_{-\infty}^\infty
  p(\xv|f, \D)p(f|\D)df
  \biggr).
\end{aligned}
\end{equation}
In the same vein, we transform the task-conditioned AF:
\begin{equation}
\begin{aligned}
  \EI_{f^\gamma}[\xv |t, \Dv] &= \int_{-\infty}^{f^\gamma}
  (f^\gamma - f)p(f|\xv, t, \Dv) df \\
  &= \int_{-\infty}^{f^\gamma}
  (f^\gamma - f)
  \frac{p(\xv, t|f, \Dv)p(f|\Dv)}{p(\xv,t|\Dv)}df \\
  &~~~~~(\because \mathrm{Bayes'~theorem}).\\
\end{aligned}
\end{equation}
To further transform, we assume the conditional shift~\cite{zhang2013domain}.
In the conditional shift, we assume that
all tasks have the same target PDF,
i.e. $p(f|t_i) = p(f|t_j)$
(see Figure~\ref{appx:detail-of-acq-fn:fig:identical-target-dist}).
Since TPE considers ranks of observations
and ignore the scale,
the optimization of $f(\xv)$
is equivalent to that of the quantile of $f(\xv)$.
For this reason, the only requirement
to achieve the conditional shift is to
use the same quantile $\gamma$ for the split
of observations $\D_m$ of all meta-tasks. 
In Figure~\ref{appx:detail-of-acq-fn:fig:identical-target-dist},
we visualize two functions
and their quantile functions along
with their target PDFs $p(f)$.
As seen in the figure, the quantile functions
equalize the target PDF
while preserving the order of each point.
Let $f$ be the quantile function (a mapping from $\X$ to $[0, 1]$) for convenience and then we obtain the following:
\begin{eqnarray}
  p(\xv, t|f, \Dv) \coloneqq \left\{
  \begin{array}{ll}
     p(\xv, t | \Dvl) & (f \leq \gamma) \\
     p(\xv, t | \Dvg) & (f > \gamma)
  \end{array}
  \right.
\end{eqnarray}
where $\Dvl \coloneqq \{\Dl_m\}_{m=1}^T$ is a set of the $\gamma$-quantile 
observations
and the conditional shift plays a role to enable us
to have a unique split point, which is at $f = \gamma$. 
Then we can further transform as follows:
\begin{align}
  &\int_{-\infty}^{f^\gamma}
  (f^\gamma - f)
  \frac{p(\xv, t|f, \Dv)p(f|\Dv)}{p(\xv,t|\Dv)}df \nonumber \\
  =& \int_{0}^{\gamma}
  (\gamma - f)
  \frac{p(\xv, t|f, \Dv)p(f|\Dv)}{p(\xv,t|\Dv)}df~(\because f \in [0, 1]) \nonumber
  \\
  =& \frac{p(\xv, t|\Dvl)}{p(\xv, t|\Dv)}
  \underbrace{\int_{0}^\gamma p(f|\Dv) df}_{
    \mathrm{const~w.r.t.~}\xv
  }\nonumber \\
  \propto &\frac{p(\xv, t|\Dvl)}{
    \gamma p(\xv, t | \Dvl) + (1 - \gamma) p(\xv, t | \Dvg) 
  }.
\end{align}

As seen in the equations above,
since TPE always ignores the improvement term
$f^\gamma - f$, we can actually view the AF as probability of improvement (PI).
Note that the equivalence of PI and EI
in the TPE formulation is already reported in
prior works~\citeappx{watanabe2022ctpe,watanabe2023ctpe,song2022general}.
Then the conditional shift can be explained in a more straightforward way.
Since PI requires the binary classification of $\D$ into
$\Dl$ or $\Dg$,
$p(f | \D)$ just depends on our split.
More specifically, 
if we regard $f$ as $\indic{f \leq f^\gamma}$
and then we simply get $p(f | \D)$
as $p(f = 1 | \D) = |\Dl|/|\D| = \gamma$
and $p(f = 0 | \D) = |\Dg|/|\D| = 1 - \gamma$.
For this reason, the assumption is easily satisfied
as long as we fix $\gamma$ for all tasks.
Note that the $\gamma$-quantile in observations is usually better, i.e. $\Dl \rightarrow \X^{\gamma^\prime}$ s.t. $\gamma^\prime < \gamma$ as stated in Theorem~\ref{main:methods:theorem:gamma-set-converges}, than the true $\gamma$-quantile
value $f^\gamma$ with a high probability when we use a non-random sampler
as a non-random sampler usually performs better than random search.
It implies that
the quality of knowledge about top domains will be relatively low in meta-tasks
if we have much more observations on the target task
and our method might not be accelerated drastically although we can guarantee the conditional shift.
We designed our task similarity to decay the similarity so that our method will not suffer from the knowledge dilution problem.

\begin{figure*}[t]
  \centering
  \includegraphics[width=0.88\textwidth]{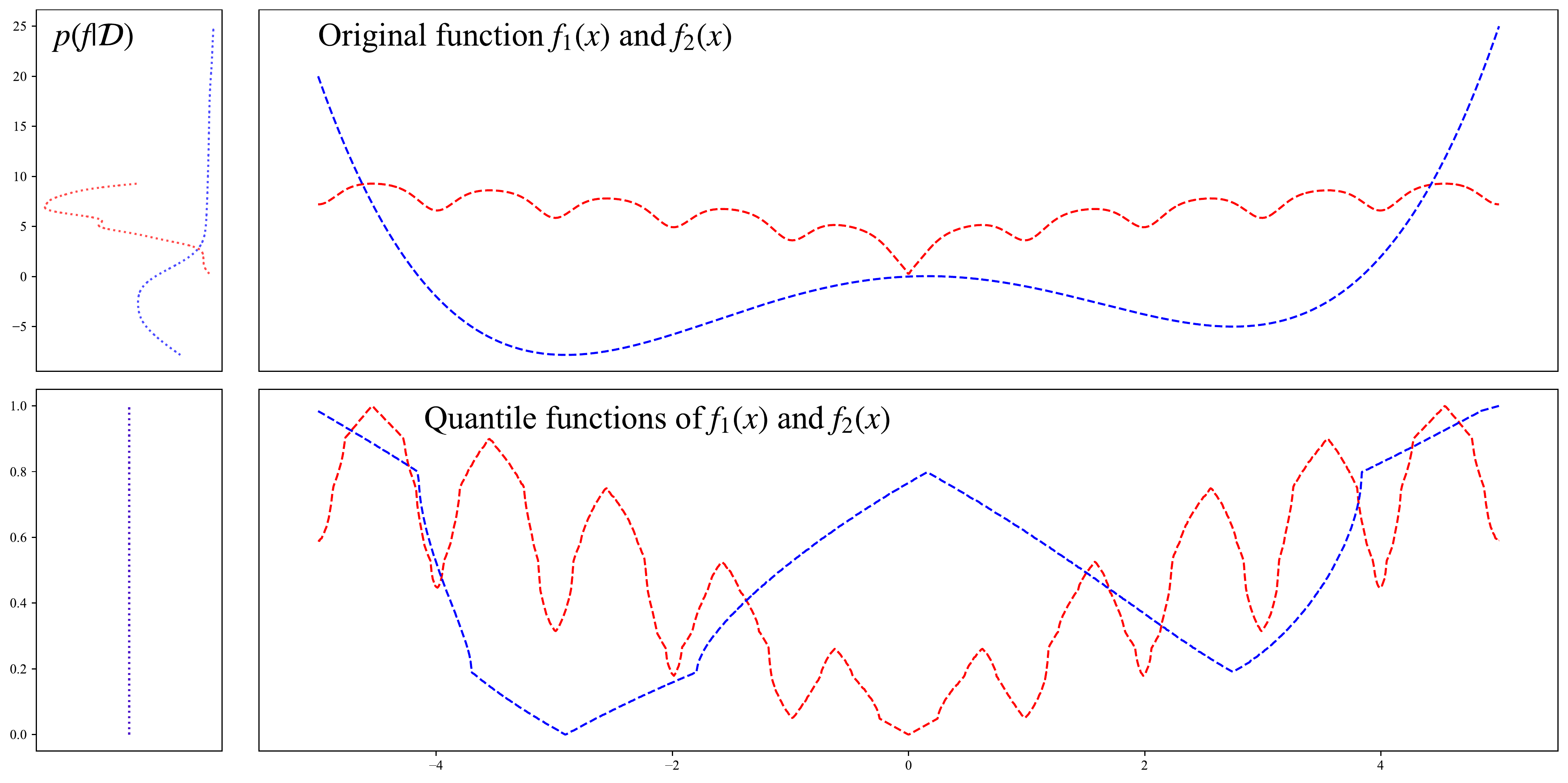}
  \caption{
    Examples of the quantile function.
    When we transform a function into a quantile function,
    we will get the uniform PDF.
    \textbf{Top row}:
    the shapes of the original functions $f_1(x), f_2(x)$ (\textbf{Right})
    and the PDFs of each function (\textbf{Left}).
    \textbf{Bottom row}:
    the shapes of the quantile functions of $f_1(x), f_2(x)$ (\textbf{Right})
    and the PDFs of each quantile function (\textbf{Left}).
    After the quantile transformation, 
    the target PDFs of each function perfectly match.
  }
  \label{appx:detail-of-acq-fn:fig:identical-target-dist}
\end{figure*}

\subsection{Task Similarity}
\label{appx:details-of-task-similarity:section}
Since the $\gamma$-set similarity is intractable due to the 
unavailability of the true $\gamma$-set $\Xg$,
we describe how to estimate the $\gamma$-set similarity.
First, we define the measure space $(\Omega, \mathcal{F}, \mu_{\Omega})$ such that
$\Omega \coloneqq \prod_{k=1}^\infty \X$, $\mathcal{F}$ is the minimum
$\sigma$-algebra on the cylinder set of $\Bx$, and
$\mu_\Omega$ is the Lebesgue measure on $\Omega$
and define the Borel body over $\X$ as $\B_{\X} \coloneqq \B_D \cap \X$.
Then, we prove the following theorem:
\begin{theorem}
  Given a measure space $(\Omega, \mathcal{F}_{\Omega}, \mu_{\Omega})$ for convergence,
  a measure space $(\X, \B_{\X}, \mu)$ for the probability measure $\prob$,
  and a strongly-consistent estimator
  of the $\gamma$-set $\mathrm{PDF}$
  $p\mathrm{:}~ \X \rightarrow \Rgeq$,
  then the following approximation
  almost surely converges to the $\gamma$-set similarity:
  \begin{equation}
  \begin{aligned}
    \hat{s}(\Dl_i, \Dl_j) \coloneqq \frac{1 - \dtv(p_i, p_j)}{1 + \dtv(p_i, p_j)}
    \stackrel{a.s.}{\rightarrow} s(\Xg_i, \Xg_j)
  \end{aligned}
  \end{equation}
  where $\stackrel{a.s.}{\rightarrow}$ implies almost sure convergence
  with respect to the measure space $(\Omega, \mathcal{F}_{\Omega}, \mu_{\Omega})$
  and $\dtv$ is the total variation distance:
  \begin{equation}
  \begin{aligned}
    \dtv(p_i, p_j) \coloneqq \frac{1}{2} \int_{\xv \in \X} |p(\xv | \Dl_i) - p(\xv|\Dl_j)| \mu(d\xv).
  \end{aligned}
  \label{appx:methods:eq:total-variation}
  \end{equation}
  \label{main:methods:theorem:convergence-of-gamma-similarity}
\end{theorem}
The proof is provided in Appendix~\ref{appx:proof-of-similarity-convergence:section}.

\section{Proofs}
\label{appx:proofs:sec}
\subsection{Assumptions}
\label{appx:assumptions:section}
In this paper, we assume the following:
\begin{enumerate}
  \item Objective function $f: \X \rightarrow \mathbb{R}$
  is Lebesgue integrable and measurable defined over
  the compact convex measurable subset $\X \subseteq \mathbb{R}^D$,
  \item The probability measure $\indic{\xv \in \Xg}$
  is Riemann integrable over $\X$,
  \item The PDF of the probability measure $\indic{\xv \in \Xg}$
  always exists, and
  \item The target PDFs are identical, i.e.
  $p(f|t_i) = p(f|t_j)$ for all task pairs,
  while the conditional PDFs can be different, i.e.
  $p(\xv|f, t_i) \neq p(\xv |f,t_j)$ where $t_i$ for $i \in \{1,\dots,T\}$
  is the symbol for
  the $i$-th task, respectively.
\end{enumerate}
Strictly speaking, we cannot guarantee that the PDF
of the probability measure $\indic{\xv \in \Xg}$ always exists;
however, we formally assume that the PDF exists by considering
the formal derivative of the step function as the Dirac delta function.
The fourth assumption is known as \textbf{conditional shift}~\citeappx{zhang2013domain}.
Note that a categorical parameter with $K$ categories is handled as $\X_i = [1, K]$
as in the TPE implementation~\citeappx{bergstra2011algorithms}
as we do not require the continuity of $f$
with respect to hyperparameters in our theoretical analysis,
this definition is valid
as long as the employed kernel
for categorical parameters 
treats different categories to be equally similar
such as Aitchison kernel~\citeappx{aitchison1976multivariate}.
In this definition, $x, x^\prime \in \X_i$
are viewed as equivalent
as long as $\lfloor x \rfloor  = \lfloor x^\prime \rfloor$
and it leads to the random sampling of each category being uniform and
the Lebesgue measure of $\X$ to be non-zero.

\subsection{Preliminaries}
\label{appx:preliminaries:section}
In this paper, we consistently use the Lebesgue integral $\int_{\xv \in \X} g(\xv)\mu(d\xv)$
instead of the Riemann integral $\int_{\xv \in \X} g(\xv)d\xv$;
however, $\int_{\xv \in \X} g(\xv)d\xv = \int_{\xv \in \X} g(\xv)\mu(d\xv)$
holds if $g(\xv)$ is Riemann integrable.
The only reason why we use the Lebesgue integral is that
some functions in our discussion cannot be handled by the Riemann integral.
Hence, we encourage readers to replace $\mu(d\xv)$ with $d\xv$
if they are not familiar to the Lebesgue integral.

Since all definitions can easily be expanded
to the MO case,
we discuss the single objective case for simplicity
and mention how to expand to the MO case in Appendix~\ref{appx:generalization-tpe:section}.
Throughout this paper, we use the following terms for the discussion later
(see Figure~\ref{main:background:fig:similarity-conceptual} to get the insight):
\begin{definition}[$\gamma$-quantile value]
  Given a quantile $\gamma \in (0, 1]$
  and a measurable function $f\mathrm{:}~ \X \rightarrow \mathbb{R}$,
  $\gamma$-quantile value $f^\gamma \in \mathbb{R}$
  is a real number such that$\mathrm{:}$
  \begin{equation}
    \begin{aligned}
      f^\gamma \coloneqq \inf\biggl\{
        f^\star \in \mathbb{R} ~\biggl|
      \int_{\xv \in \X} \indic{f(\xv) \leq f^\star}
      \frac{\mu(d\xv)}{\mu(\X)} \geq \gamma
      \biggr\}.
    \end{aligned}
  \end{equation}
\end{definition}
In this paper, if $\X$ is a countable set,
we only consider discrete dimensions
such that $\{\frac{k}{n}\}_{k=0}^{n}$ ($n$ can be infinity).
For example, when $D = 1$,
we define $I_{k,n} \coloneqq [\frac{k}{n}, \frac{k+1}{n})$
and $f(x) \coloneqq \sum_{k=0}^{n-1} \indic{x \in I_{k,n}} f(\frac{k}{n})$
(if $n \rightarrow \infty$, we will take the limit of the RHS).
Then we can relax a discrete function to a continuous function,
and thus we obtain the $\gamma$-quantile value as defined above.
\begin{definition}[$\gamma$-set]
  Given a quantile $\gamma \in (0, 1]$,
  $\gamma$-set $\Xg$ is defined as
  $\Xg \coloneqq \{\xv \in \X \mid f(\xv) \leq f^\gamma\} \in \Bx$
  where $\Bx \coloneqq \B_D \cap \X$ is the Borel body over $\X$.
  \label{main:background:def:gamma-set}
\end{definition}

\begin{definition}[$\gamma$-set similarity]
  Given two objective functions $f_1, f_2$
  and $\gamma \in (0, 1]$,
  let the $\gamma$-sets for $f_1, f_2$ be
  $\Xg_1, \Xg_2$.
  Then we define $\gamma$-set similarity
  $s: \Bx \times \Bx \rightarrow [0, 1]$
  between $f_1$ and $f_2$ as follows$\mathrm{:}$
  \begin{equation}
    \begin{aligned}
      s(\Xg_1, \Xg_2) = \frac{\mu(\Xg_1 \cap \Xg_2)}{\mu(\Xg_1 \cup \Xg_2)}.
    \end{aligned}
  \end{equation}
\end{definition}
Note that those definitions rely on some assumptions described in Appendix~\ref{appx:assumptions:section}.
\subsection{Generalization of Tree-Structured Parzen Estimator with Multi-Objective Optimization}
\label{appx:generalization-tpe:section}
As mentioned in the main paper,
the extension of TPE to MO settings can be easily
achieved using a rank metric $R$.
For this reason, we stick to the single objective
setting for simplicity;
however, we describe the theoretical background
for the extension here.
We first reformulate PI
(as mentioned previously the AF of TPE is equivalent to PI~\cite{watanabe2022ctpe,watanabe2023ctpe,song2022general})
using the Lebesgue integral:
\begin{equation}
\begin{aligned}
  &\mathrm{Riemann~formulation}\!: \\
  &~~~~~~~~~~\prob(f \leq f^\gamma | \xv, \D) = \int_{-\infty}^{f^\gamma} p(f|\xv,\D) df \\
  &\mathrm{Lebesgue~formulation}\!: \\
  &~~~~~~~~~~\prob(f \leq  f^\gamma | \xv, \D) = \int_{\Y} p(f|\xv,\D) \indic{f \in \Y^\gamma} \mu_{\Y}(df)
\end{aligned}
\end{equation}
where $\Y^\gamma \in \B_{\Y}$ is defined as
the Borel body $\B_{\Y}$ over the objective space $\Y$
and $\mu_{\Y}$ is the Lebesgue measure defined on $\Y$, and
$\Y^\gamma$ is defined so that:
\begin{equation}
\begin{aligned}
  \gamma \coloneqq 
  \frac{
    \int_{\Y} p(f|\xv,\D) \indic{f \in \Y^\gamma} \mu_{\Y}(df)
  }{
    \int_{\Y} p(f|\xv,\D) \mu_{\Y}(df)
  }
\end{aligned}
\end{equation}
holds.
For example, $\Y^\gamma = (-\infty, f^\gamma]$ in the 1D example
and $\Y^\gamma = \{ f \in \Y \mid R(f) \leq R(f^\gamma) \}$
in MO cases.
In MO cases, only the Lebesgue formulation could be strictly
defined and we compute PI as follows:
\begin{equation}
\begin{aligned}
  \prob(\fv \preceq  \fv^\gamma | \xv, \D)
  &\coloneqq
  \prob(R(\fv) \leq  R(\fv^\gamma) | \xv, \D) \\
  &=
  \int_{\Y} p(\fv|\xv,\D) \indic{\fv \in \Y^\gamma} \mu_{\Y}(d\fv)
\end{aligned}
\end{equation}
where $\preceq$ is an operator
such that $\fv$ is equally good
or better than $\fv^\gamma$
based on a rank metric $R: \mathbb{R}^m \rightarrow \mathbb{R}$
for an objective vector $\fv$.
For example,
in the single objective case,
$\Y^\gamma$ is just $(-\infty, f^\gamma]$
and $R(f) = f$.
In MO cases,
we used the non-dominated rank
and the crowding distance as $R$.
In principle, MO settings fall back to the single objective setting
after we apply $R$ to an objective vector $\fv$.

\subsection{Proof of Theorem~\ref{main:methods:theorem:convergence-of-gamma-similarity}}
\label{appx:proof-of-similarity-convergence:section}
Throughout this section, we denote $\Dl_{m, N}$ as the $\Dl_m$ with the size of $N$
for notational clarity.
Then we first prove the following lemma:
\begin{lemma}
  Given the a measurable space $(\X, \Bx, \mu)$, the $\gamma$-set similarity
  between two functions $f_i, f_j$ is computed as$\mathrm{:}$
  \begin{equation}
    \begin{aligned}
      s(\Xg_i, \Xg_j) = \frac{
        1 - \dtv(p_i, p_j)
      }{
        1 + \dtv(p_i, p_j)
      }.
    \end{aligned}
  \end{equation}
  where
  $P_i = \mathbb{P}[\xv \in \Xg_i], P_j = \mathbb{P}[\xv \in \Xg_j]$
  are the probability measure of $p_i = p(\xv|\Xg_i), p_j = p(\xv|\Xg_j)$
  and
  \begin{equation}
    \begin{aligned}
      \dtv(p_i, p_j) &= \frac{1}{2}\|P_i - P_j \|_1 \\
      &= \frac{1}{2}\int_{\xv \in \X} | p(\xv | \Xg_i) - p(\xv | \Xg_j) |\mu(d\xv)
    \end{aligned}
    \label{appx:proofs:eq:total-variation}
  \end{equation}
  is the total variation distance.
  \label{appx:proofs:lemma:similarity-and-total-variation}
\end{lemma}
Note that we used
the Scheffe's lemma~\citeappx{tsybakovintroduction}
for the transformation in Eq.~(\ref{appx:proofs:eq:total-variation}).
\begin{proof}
  Since $\xv$ is either in or not in $\Xg$
  and $\Xg$ is fixed,
  the probability measure takes either $0$ or $1$.
  For this reason, the following holds$\mathrm{:}$
  \begin{equation}
    \begin{aligned}
      \indic{\xv \in \Xg_i \cap \Xg_j} &= \mathbb{P}[\xv \in \Xg_i \cap \Xg_j] \\
      & = \mathbb{P}[\xv \in \Xg_i] \wedge \mathbb{P}[\xv \in \Xg_j], \\
      \indic{\xv \in \Xg_i \cup \Xg_j} &= \mathbb{P}[\xv \in \Xg_i \cup \Xg_j] \\
      & = \mathbb{P}[\xv \in \Xg_i] \vee \mathbb{P}[\xv \in \Xg_j]
    \end{aligned}
    \label{appx:proof:eq:prob-measure-and-operators}
  \end{equation}
  where $\wedge, \vee$ are equivalent to the $\min$ and $\max$ operator, respectively.
  Notice that $\Xg_1, \Xg_2$ must be measurable sets so that the equation above is valid.
  It leads to the following equations$\mathrm{:}$
  \begin{equation}
    \begin{aligned}
      \mu(\Xg_i \cap \Xg_j) & = \int_{\xv \in \X} \mathbb{P}[\xv \in \Xg_i \cap \Xg_j] \mu(d\xv), \\
      \mu(\Xg_i \cup \Xg_j) & = \int_{\xv \in \X} \mathbb{P}[\xv \in \Xg_i \cup \Xg_j] \mu(d\xv). \\
    \end{aligned}
    \label{appx:proof:eq:volume-vs-prob-measure}
  \end{equation}
  Using 
  $\mathrm{Eqs.~(\ref{appx:proof:eq:prob-measure-and-operators}), (\ref{appx:proof:eq:volume-vs-prob-measure})}$,
  we obtain the following equation$\mathrm{:}$
  \begin{equation}
    \begin{aligned}
      s(\Xg_i, \Xg_j) &= \frac{
        \int_{\xv \in \X} \mathbb{P}[\xv \in \Xg_i] \wedge \mathbb{P}[\xv \in \Xg_j] \mu(d\xv)
      }{
        \int_{\xv \in \X} \mathbb{P}[\xv \in \Xg_i] \vee \mathbb{P}[\xv \in \Xg_j] \mu(d\xv)
      } \\
      &= \frac{
        1 - \dtv(p_i, p_j)
      }{
        1 + \dtv(p_i, p_j)
      }.
    \end{aligned}
    \label{appx:proofs:eq:true-task-similarity}
  \end{equation}
  This completes the proof.
  \label{appx:proofs:proof:similarity-is-total-variation-ratio}
\end{proof}

Then we need to prove the following lemma as well:
\begin{lemma}
  Given the assumptions above and the measurable spaces defined in the statement of $\mathrm{Theorem}~\ref{main:methods:theorem:convergence-of-gamma-similarity}$,
  the following holds under the assumption (denote $\mathrm{Assumption~(\star)}$) of
  $\exists \epsilon > 0,
    \lim_{N\rightarrow \infty}
    \|(p(\xv|\Dl_{m,N}) - \indic{\xv \in \Xg_m} / \mu(\Xg_m))
    \indic{\xv \in \epsilon \boldsymbol{B}_N}
    \|_{\infty}
    = 0$
  for all $m$
  where $\boldsymbol{B}_N \subseteq \mathbb{R}^D$
  is a ball with radius $N$ centered at $\boldsymbol{0}\mathrm{:}$
  \begin{equation}
    \begin{aligned}
      &\int_{\xv \in \X} | p(\xv | \Dl_{i,N}) - p(\xv | \Dl_{j,N}) | \mu(d\xv) \\
      &\stackrel{a.s.}{\rightarrow}
      \int_{\xv \in \X} \biggl|
      \frac{\indic{\xv \in \Xg_i}}{\mu(\Xg_i)} -
      \frac{\indic{\xv \in \Xg_j}}{\mu(\Xg_j)}
      \biggr| \mu(d\xv) ~ (N \rightarrow \infty).
    \end{aligned}
    \label{appx:proofs:eq:convergence-of-total-variation}
  \end{equation}
  \label{appx:proofs:lemma:convergence-of-total-variation}
\end{lemma}

\begin{proof}
  From the assumption of the strongly-consistent estimator,
  the following uniform convergence holds~\citeappx{wied2012consistency}
  for all $m \in \{1,\dots,T\}$ and $\forall \epsilon > 0$
  at each point of continuity $\xv$ of $\indic{\xv \in \Xg_m}\mathrm{:}$
  \begin{equation}
    \begin{aligned}
      p(\xv | \Dl_{m,N})
      \stackrel{a.s.}{\rightarrow}
      \frac{\indic{\xv \in \Xg_m}}{\mu(\Xg_m)}~(N \rightarrow \infty).
    \end{aligned}
  \end{equation}
  Note that the uniform convergence almost surely happens
  with respect to an arbitrary sample from $\Omega$ that achieves the strong consistency of $\mathrm{KDE}$.
  Therefore, the following holds using the continuous mapping theorem$\mathrm{:}$
  \begin{equation}
    \begin{aligned}
      & | p(\xv | \Dl_{i,N}) - p(\xv | \Dl_{j,N}) | \\
      & \stackrel{a.s.}{\rightarrow}
      \biggl|
      \frac{\indic{\xv \in \Xg_i}}{\mu(\Xg_i)} -
      \frac{\indic{\xv \in \Xg_j}}{\mu(\Xg_j)}
      \biggr|
      ~(N \rightarrow \infty).
    \end{aligned}
  \end{equation}
  From the assumptions, all functions are defined
  over the compact convex measurable subset $\X \subseteq \mathbb{R}^D$
  and all convergences in this proof are uniform
  and almost sure;
  therefore, we obtain the following result based on
  $\mathrm{Theorem~7.16}$ of \citeappx{rudin1976principles}$\mathrm{:}$
  \begin{equation}
    \begin{aligned}
      &\lim_{N\rightarrow \infty}\int_{\xv \in \X} |p(\xv | \Dl_{i,N}) - p(\xv | \Dl_{j,N})|
      \mu(d\xv) \\
       & =\int_{\xv \in \X} \lim_{N\rightarrow \infty} |p(\xv | \Dl_{i,N}) - p(\xv | \Dl_{j,N})|
      \mu(d\xv)\\
       & \stackrel{a.s.}{\rightarrow}
      \int_{\xv \in \X}\biggl|
      \frac{\indic{\xv \in \Xg_i}}{\mu(\Xg_i)} -
      \frac{\indic{\xv \in \Xg_j}}{\mu(\Xg_j)}
      \biggr| \mu(d\xv).
    \end{aligned}
  \end{equation}
  Note that the last transformation requires $\mathrm{Assumption~(\star)}$
  to be true with the probability of $1$.
  This completes the proof.
\end{proof}

From Eq.~(\ref{appx:proofs:eq:true-task-similarity}) in
Lemma~\ref{appx:proofs:lemma:similarity-and-total-variation},
Eq.~(\ref{appx:proofs:eq:convergence-of-total-variation}) in
Lemma~\ref{appx:proofs:lemma:convergence-of-total-variation}
and the continuous mapping theorem,
the statement of Theorem~\ref{main:methods:theorem:convergence-of-gamma-similarity}
is proved and this completes the proof.

\subsection{Proof of Theorem~\ref{main:methods:theorem:gamma-set-converges}}
\label{appendix:proofs:subsection:proof-of-gamma-set-convergence}
\begin{theorem}
  If we use the $\epsilon$-greedy policy ($\epsilon \in (0, 1)$) for $\mathrm{TPE}$ to choose
  the next candidate $\xv$ for a noise-free objective function $f(\xv)$
  defined on search space $\X$ with at most a countable number of configurations,
  and we use a $\mathrm{KDE}$ whose distribution converges to the empirical distribution as
  the number of samples goes to infinity,
  then a set of the top-$\gamma$-quantile observations $\Dl_1$ is almost surely
  a subset of $\Xg$.
  \label{appendix:proofs:theorem:tpe-with-eps-greedy-converges-to-gamma-set}
\end{theorem}

\begin{proof}
  First, we define a function defined on at most countable set
  as $f_k(\xv) \coloneqq f(k\lfloor \xv / k \rfloor)$
  so that we can still define the $\mathrm{PDF}$ $p(\xv|f)$ on $\X$.
  In this proof, we assume $\X = [0,1]^D$ and
  we choose the next configuration from $\X_{(k)} \coloneqq \{0,1/k,\dots,(k-1)/k,1\}^D$.
  We first prove the finite case and 
  then take $k \rightarrow \infty$ to prove the countable case.
  First, we list the facts that
  are immediately drawn from the assumptions$\mathrm{:}$
  \begin{enumerate}
    \item The $\mathrm{KDE}$ with a set of control parameters $\alphav_N$
    at the $N$-th iteration converges to the empirical distribution, i.e.
    \begin{equation}
    \begin{aligned}
      \lim_{N \rightarrow \infty} \frac{1}{N} \sum_{n^\prime=1}^N k(\xv, \xv_{n^\prime} | \alphav_N)
      \rightarrow
      \lim_{N \rightarrow \infty} \frac{1}{N} \sum_{n^\prime=1}^N \delta(\xv, \xv_{n^\prime}), 
    \end{aligned}
    \end{equation}
    \item Due to the random part in the $\epsilon$-greedy policy, we almost surely covers
    all possible configurations $\xv \in \X_{(k)}$ with a sufficiently large number of samples,
    \item $\forall \xv \in \Dl_N, p(\xv|\Dl_N) > 0$,
    \item $\forall \epsilon^\prime > 0, \exists N > 0, \forall \xv \notin \Dl_N, p(\xv|\Dl_N) < \epsilon^\prime$,
    \item $\forall \xv \in \Dg_N, p(\xv|\Dg_N) > 0$, and
    \item $\forall \epsilon^\prime > 0, \exists N > 0, \forall \xv \notin \Dg_N, p(\xv|\Dg_N) < \epsilon^\prime$.
  \end{enumerate}
  Since $\forall \epsilon^\prime > 0, \exists N > 0, \Dl_N \stackrel{a.s.}{\subseteq} \X_{(k)}^{\gamma + \epsilon^\prime}$
  holds from the weak law of large numbers if $\epsilon = 1$,
  the main strategy of this proof is to show
  ``\textbf{the probability of drawing $\xv \in \Xg_{(k)}$ is higher than $\gamma$ as $N$ goes to infinity}''.
  We use the following fact in the $\mathrm{AF}$
  of the $\mathrm{TPE}$ formulation~\citeappx{watanabe2022ctpe,watanabe2023ctpe}$\mathrm{:}$
  \begin{equation}
  \begin{aligned}
    \mathrm{EI}_{f^\gamma}[\xv] \propto \mathrm{PI}_{f^\gamma}[\xv]
    = \frac{\gamma p(\xv|\Dl)}{\gamma p(\xv|\Dl) + (1 - \gamma) p(\xv | \Dg)}.
  \end{aligned}
  \end{equation}
  Since we use the $\epsilon$-greedy algorithm,
  $\lim_{N \rightarrow \infty} p(\xv | \Dl_N) \stackrel{a.s.}{\geq} \epsilon / (k + 1)^D$
  and 
  $\lim_{N \rightarrow \infty} p(\xv | \Dg_N) \stackrel{a.s.}{\geq} \epsilon / (k + 1)^D$
  hold for all $\xv \in \X_{(k)}$.
  It implies that $p(\xv | \Dl_N)$ and $p(\xv | \Dg_N)$
  dominate each other if $\xv \notin \Dg_N$ or $\xv \notin \Dl_N$,
  respectively.
  Therefore, when we take a sufficiently large $N$,
  the following holds$\mathrm{:}$
  \begin{enumerate}
    \item $\forall \epsilon^\prime, \exists N > 0, 
    \forall \xv \notin \Dg_N, |\mathrm{PI}_{f^\gamma}[\xv] - 1 | < \epsilon^\prime$, and
    \item $\forall \epsilon^\prime, \exists N > 0, 
    \forall \xv \notin \Dl_N, \mathrm{PI}_{f^\gamma}[\xv] < \epsilon^\prime$.
  \end{enumerate}
  Notice that we cannot determine $\mathrm{PI}_{f^\gamma}[\xv]$
  for $\xv \in \Dl_N \cap \Dg_N$
  and we assume that a sufficiently large $N$ allows $\mathrm{TPE}$ to
  cover all possible configurations at least once.
  For this reason, the greedy algorithm always picks a configuration
  $\xv \in \Dl_N$.
  Since $\Dl_N$ is determined based on the objective function $f_k$
  and all possible configurations are covered
  due to the random policy,
  either of the following is satisfied$\mathrm{:}$
  \begin{enumerate}
    \item $\Dl_N \subseteq \X_{(k)}^{\gamma^\prime}$ such that $\gamma^\prime \leq \gamma$, or
    \item $\X_{(k)}^\gamma \subseteq \Dl_N$.
  \end{enumerate}
  For the first case, we obtain $\xv \in \Xg_{(k)}$ with the probability of $1$
  and we obtain $\xv \in \Xg_{(k)}$ with the probability of
  $\gamma (k+1)^D /|\set{\Dl_N}| > \gamma$ for the second case.
  Since the random policy obviously picks a configuration $\xv \in \Xg_{(k)}$
  with the probability of $\gamma$,
  the probability of drawing $\xv \in \Xg_{(k)}$ is larger than
  $\epsilon \gamma + (1 - \epsilon) \min(1, \gamma (k+1)^D/|\set{\Dl_N}|) > \gamma$;
  therefore, it concludes that
  $\Dl_N \stackrel{a.s.}{\subseteq} \Xg_{(k)}$ and 
  this completes the proof for the finite case.
  The countable case is completed by $n \rightarrow \infty$.
\end{proof}

Using this lemma, we prove the theorem of interest, i.e.
Theorem~\ref{main:methods:theorem:gamma-set-converges}.
\begin{proof}
  Let $\{\D_m\}_{m=1}^T$
  with $|\D_1| = N$ be $\Dv_N$
  and the splits of $\Dv_N$
  be $\Dvl_N$ and $\Dvg_N$.
  We also define the observations of the target task with
  the size of $N$ as $\D_N$ for the simplicity.
  Notice that we assume that $N$ is sufficiently large
  to cover all possible configurations as in the proof of
  $\mathrm{Theorem~\ref{appendix:proofs:theorem:tpe-with-eps-greedy-converges-to-gamma-set}}$. 
  Due to the fact that
  the limit of the joint $\mathrm{PDF}$
  converges to the original $\mathrm{PDF}$, i.e.
  $\lim_{N \rightarrow \infty} p(\xv,t|\Dvl_N) = p(\xv | \Dl_N)$
  and $\lim_{N \rightarrow \infty} p(\xv,t|\Dvg_N) = p(\xv | \Dg_N)$,
  the following holds$\mathrm{:}$
  \begin{equation*}
  \begin{aligned}
    &\mathrm{Property}~(1).\\
    &~~~~~\exists N > 0, \forall \xv \in \Dl_N, p(\xv,t|\Dvl_N) > 0, \\
    &\mathrm{Property}~(2).\\
    &~~~~~\forall \epsilon^\prime > 0, \exists N > 0, \forall \xv \in \Dl_N \setminus \Dg_N, 
    p(\xv,t|\Dvg_N) < \epsilon^\prime, \\
    &\mathrm{Property}~(3).\\
    &~~~~~\exists N > 0, \forall \xv \in \Dg_N, p(\xv,t|\Dvg_N) > 0, \\
    &\mathrm{Property}~(4).\\
    &~~~~~\forall \epsilon^\prime > 0, \exists N > 0, \forall \xv \in \Dg_N \setminus \Dl_N, 
    p(\xv,t|\Dvl_N) < \epsilon^\prime. \\
  \end{aligned}
  \end{equation*}
  From $\mathrm{Properties~(1)}$ and $\mathrm{(2)}$,
  the $\mathrm{AF}$ converges to $1$
  if $\xv \in \Dl_N\setminus\Dg_N, N \rightarrow \infty$.
  From $\mathrm{Properties~(3)}$ and $\mathrm{(4)}$,
  the $\mathrm{AF}$ converges to 0
  if $\xv \in \Dg_N\setminus \Dl_N, N \rightarrow \infty$.  
  Furthermore,
  the $\mathrm{AF}$ may take non-zero
  if $\forall \xv \in \Dl_N \cap \Dg_N$.
  Since those conditions are the same for
  $\mathrm{Theorem~\ref{appendix:proofs:theorem:tpe-with-eps-greedy-converges-to-gamma-set}}$,
  the same discussion applies to this theorem,
  and thus this completes the proof.
\end{proof}

\subsection{Proof of Proposition \ref{main:methods:proposition:trivial-dims-not-matter}}
\label{appendix:proofs:subsection:proof-of-trivial-dims-not-matter}
We first define $\mu_{\X}$ as the the Lebesgue measure on $\X$,
$\mu_{d}$ as the Lebesgue measure on $\X_d$,
and $\mu_{-d}$ as the Lebesgue measure on $\X_{-d}$.
Then we formally define the trivial dimensions:
\begin{definition}[Trivial dimension]
  Given a quantile $\gamma \in (0, 1]$
  and a measure space $(\X_d, \B_d \cap \X_d, \mu_d)$,
  the $d$-th dimension is trivial if
  $\prob[\xv \in \Xg | \xv_{-d}] \in \{0, 1\}$
  where $\xv_{-d}$ is $\xv$ without the $d$-th dimension and
  $\X_d$ is the domain of the $d$-th dimension.
\end{definition}
The definition implies the $d$-th dimension
is trivial if and only if the binary function
$b(\xv | \Xg, \xv_{-d}) \coloneqq \indic{\xv \in \Xg | \xv_{-d}}$
 does not change
almost everywhere ($\mu_d$-a.e.) due to the variation in the $d$-th dimension
when the other dimensions are fixed
where we define $b(\xv | \Xg, \xv_{-d})$ as $b(\xv | \Xg)$ such that
each element, except for the $d$-th dimension,
of $\xv$ is fixed to $\xv_{-d}$.
Note that trivial dimension
is a stronger concept than
less important dimension in f-ANOVA.
More specifically, f-ANOVA computes the marginal variance
of each dimension
and judges important dimensions based on the variance;
however, zero variance, which implies the dimension is
not important, does not guarantee that such a dimension
is trivial.
For example, $f(x, y) = \mathrm{sign}(x) \sin y$ leads
to zero variance in both dimensions, but both dimensions
obviously change the function values.

Then we prove Proposition~\ref{main:methods:proposition:trivial-dims-not-matter}.
\begin{proof}
  From the definition of the trivial dimension,
  the following holds almost surely for $x_d \in \X_d$ when we assume the $d$-th dimension
  is trivial$\mathrm{:}$
  \begin{equation}
  \begin{aligned}
    \int_{\xv_{-d} \in \X_{-d}} b(\xv|\Xg, \xv_{-d})\mu_{-d}(d\xv_{-d})
    =
    \frac{\mu_{\X}(\Xg)}{\mu_d(\X_d)}\biggl\vert_{x_d} ~\mu_d\text{-}a.e.,
  \end{aligned}
  \end{equation}
  where $\mid_{x_d}$ means that the $d$-th dimension in $\xv$ is $x_d$.
  Additionally, as the $d$-th dimension of two tasks
  are trivial, we obtain the following for $\forall \xv_{-d} \in \X_{-d}\mathrm{:}$
  \begin{equation}
  \begin{aligned}
    b(\xv | \Xg_1 \cap \Xg_2, \xv_{-d}) &= \min(b(\xv | \Xg_1, \xv_{-d}), b(\xv | \Xg_2, \xv_{-d})),\\
    b(\xv | \Xg_1 \cup \Xg_2, \xv_{-d}) &= \max(b(\xv | \Xg_1, \xv_{-d}), b(\xv | \Xg_2, \xv_{-d})) \\
  \end{aligned}
  \end{equation}
  where $\Xg_1, \Xg_2$ are the $\gamma$-sets of tasks 1, 2.
  As both $b(\xv | \Xg_1, \xv_{-d}), b(\xv | \Xg_2, \xv_{-d})$
  are also almost everywhere $0$ or $1$ ($\mu_d$-a.e.),
  $b(\xv | \Xg_1 \cup \Xg_2, \xv_{-d}), b(\xv | \Xg_1 \cap \Xg_2, \xv_{-d})$
  are almost everywhere $0$ or $1$.
  Therefore, we obtain$\mathrm{:}$
  \begin{equation}
  \begin{aligned}
    &\int_{\xv_{-d} \in \X_{-d}} b(\xv | \Xg_1 \cap \Xg_2, \xv_{-d})\mu_{-d}(d\xv_{-d}) \\
    &~~~~~=
    \frac{\mu_{\X}(\Xg_1 \cap \Xg_2)}{\mu_{d}(\X_d)}\biggl\vert_{x_d} ~\mu_d\text{-}a.e., \\
    &\int_{\xv_{-d} \in \X_{-d}} b(\xv | \Xg_1 \cup \Xg_2, \xv_{-d})\mu_{-d}(d\xv_{-d}) \\
    &~~~~~=
    \frac{\mu_{\X}(\Xg_1 \cup \Xg_2)}{\mu_{d}(\X_d)}\biggl\vert_{x_d}~\mu_d\text{-}a.e.. \\
  \end{aligned}
  \label{appx:proofs:eq:intersec-union-is-const}
  \end{equation}
  Using these results, the following holds$\mathrm{:}$
  \begin{equation}
  \begin{aligned}
    s(\Xg_1, \Xg_2) &= 
    \frac{\mu_{\X}(\Xg_1 \cap \Xg_2)}{\mu_{\X}(\Xg_1 \cup \Xg_2)} \\
    &= \frac{
      \int_{\xv_{-d} \in \X_{-d}} b(\xv | \Xg_1 \cap \Xg_2, \xv_{-d})\mu_{-d}(d\xv_{-d})
    }{
      \int_{\xv_{-d} \in \X_{-d}} b(\xv | \Xg_1 \cup \Xg_2, \xv_{-d})\mu_{-d}(d\xv_{-d})
    }\\
    &= \frac{
      \int_{\xv \in \X} b(\xv | \Xg_1 \cap \Xg_2)\mu_{\X}(d\xv)
    }{
      \int_{\xv \in \X} b(\xv | \Xg_1 \cup \Xg_2)\mu_{\X}(d\xv)
    }
  \end{aligned}
  \end{equation}
  where the last transformation
  used $\mathrm{Eq.~(\ref{appx:proofs:eq:intersec-union-is-const})}$
  and this completes the proof.
\end{proof}

\begin{figure*}[t]
  \centering
  \includegraphics[width=0.98\textwidth]{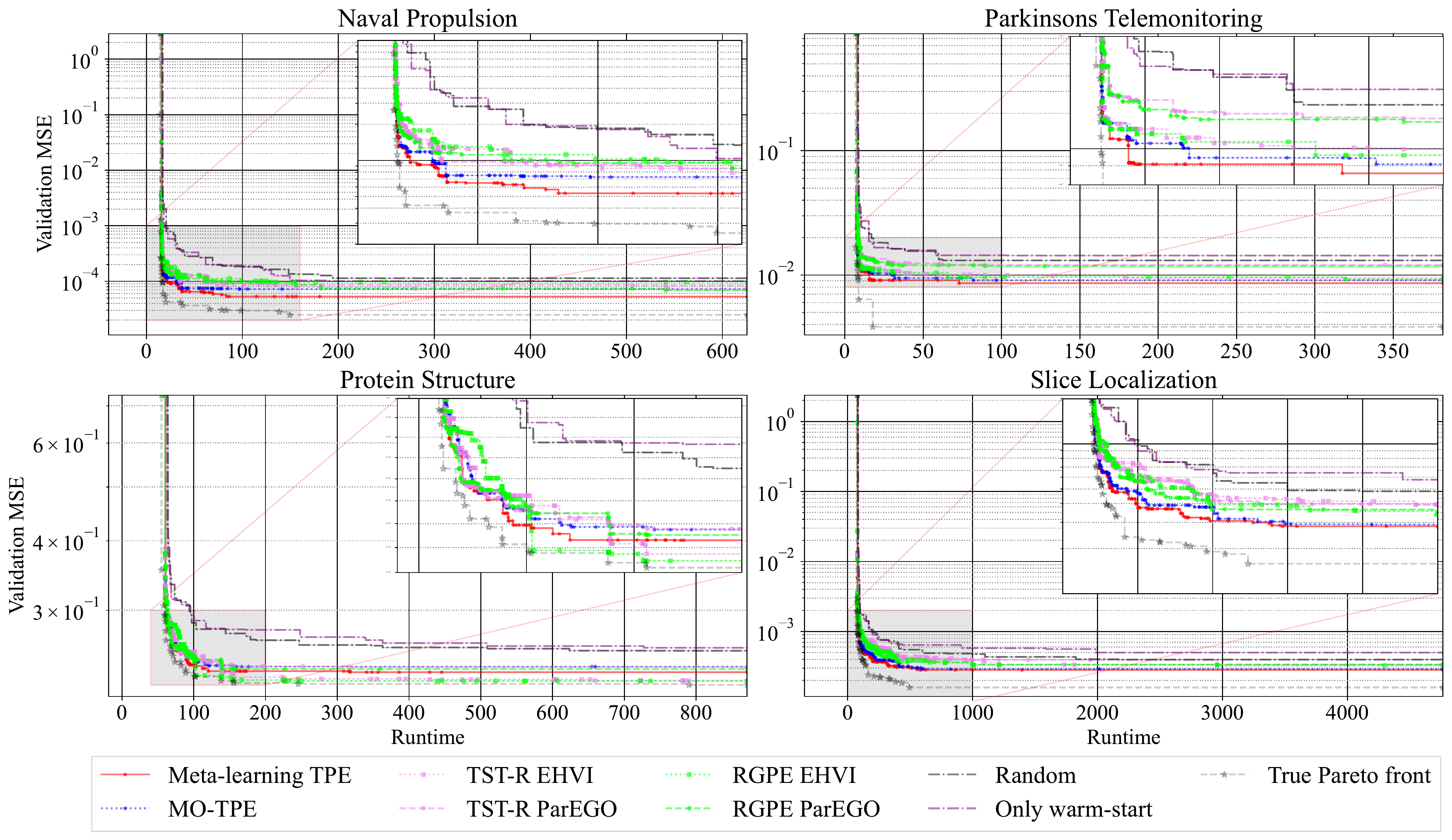}
  \caption{
    The $50\%$ empirical attainment surfaces~\protect\citeappx{watanabe2023pareto}
    of each method
    over $20$ different random seeds
    on HPOlib.
    The objectives are to minimize runtime and validation MSE.
  }
  \label{appx:additional-results:fig:eaf-hpolib}
\end{figure*}

\begin{figure*}[t]
  \centering
  \includegraphics[width=0.98\textwidth]{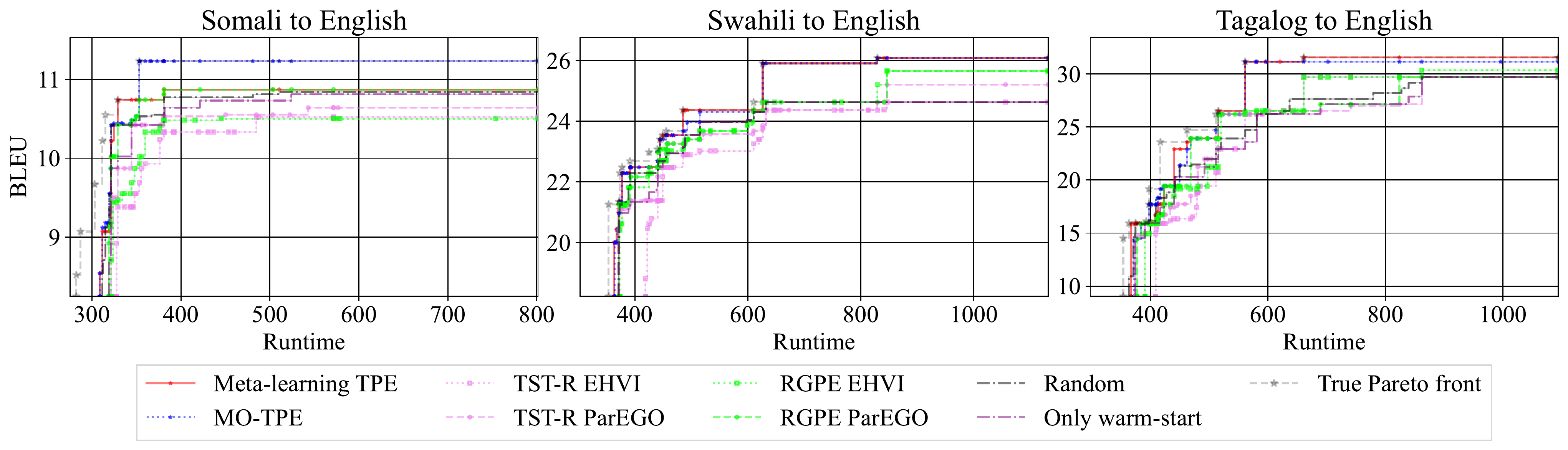}
  \caption{
    The $50\%$ empirical attainment surfaces~\protect\citeappx{watanabe2023pareto} of each method
    over $20$ different random seeds
    on NMT-Bench.
    The objectives are to minimize runtime
    and maximize BLEU.
  }
  \label{appx:additional-results:fig:eaf-nmt}
\end{figure*}

\begin{figure*}[t]
  \centering
  \includegraphics[width=0.98\textwidth]{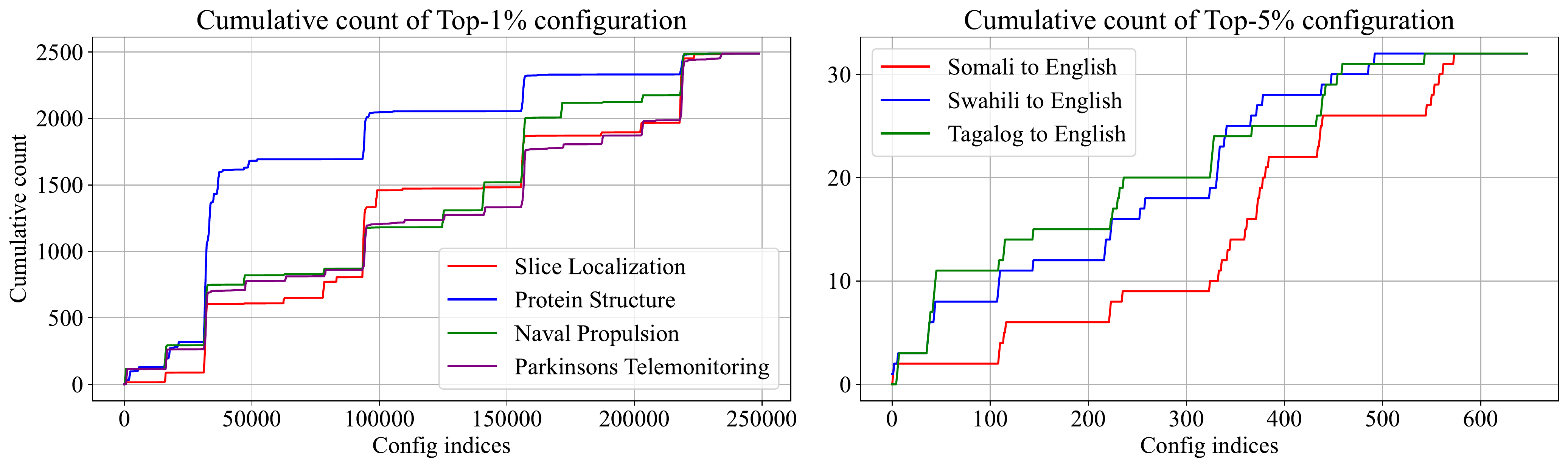}
  \caption{
    The cumulative count of top-$1\%$ on HPOlib (\textbf{Left})
    and top-$5\%$ on NMT-Bench (\textbf{Right}) over configuration indices.
    Although subsequent indices do not necessarily imply
    that those configurations are close in the search space,
    each index identifies the same configuration for each dataset,
    and thus we can see trends by tracking the variations of each curve.
  }
  \label{appx:additional-results:fig:dataset-dist}
\end{figure*}

\section{Details of Experiments}
\label{appx:details-of-experiments:section}
\addtolength{\tabcolsep}{-5pt}
\begin{table}
  \begin{center}
    \caption{
      The search space of HPOlib.
      All tasks have the same search space by default.
      Each benchmark has performance metrics of $62208$
      possible configurations with $4$ random seeds.
    }
    \vspace{-2mm}
    \label{appendix:details-of-experiments:tab:hpolib}
    \begin{tabular}{ll}
      \toprule
      Hyperparameter          & Choices                                         \\
      \midrule
      Number of units 1       & \{$2^4, 2^5, 2^6, 2^7, 2^8, 2^9$\}              \\
      Number of units 2       & \{$2^4, 2^5, 2^6, 2^7, 2^8, 2^9$\}              \\
      Dropout rate 1          & \{$0.0,0.3,0.6$\}                               \\
      Dropout rate 2          & \{$0.0,0.3,0.6$\}                               \\
      Activation function 1   & \{ReLU$,$ tanh\}                                \\
      Activation function 2   & \{ReLU$,$ tanh\}                                \\
      Batch size              & \{$2^3, 2^4, 2^5, 2^6$\}                        \\
      Learning rate scheduler & \{cosine$,$ constant\}                          \\
      Initial learning rate   & \{$5$e-$4,1$e-$3,5$e-$3,1$e-$2,5$e-$2,1$e-$1$\} \\
      \bottomrule
    \end{tabular}
  \end{center}
\end{table}

\begin{table}
  \begin{center}
    \caption{
      The search space of NMT-Bench.
      To make the search spaces for all tasks identical,
      we reduced the choices of ``Number of layers''
      in the ``Swahili to English'' dataset.
      Each benchmark has performance metrics of $648$
      possible configurations with $1$ random seed.
    }
    \vspace{-2mm}
    \label{appendix:details-of-experiments:tab:nmt-bench}
        \begin{tabular}{ll}
          \toprule
          Hyperparameter                           & Choices                            \\
          \midrule
          Number of BPE symbols ($\times 10^3$)    & \{$2^0, 2^1, 2^2, 2^3, 2^4, 2^5$\} \\
          Number of layers                         & \{$1,2,4$\}                        \\
          Embedding size                           & \{$2^8, 2^9, 2^{10}$\}             \\
          Number of hidden units in each layer     & \{$2^{10}, 2^{11}$\}               \\
          Number of heads in self-attention        & \{$2^3, 2^4$\}                     \\
          Initial learning rate ($\times 10^{-4}$) & \{$3, 6, 10$\}                     \\
          \bottomrule
        \end{tabular}
  \end{center}
\end{table}
\addtolength{\tabcolsep}{5pt}

\subsection{Dataset Description}
In the experiments, we used the following tabular benchmarks:
\begin{enumerate}
  \item HPOlib (Slice Localization,
  Naval Propulsion,
  Parkinsons Telemonitoring,
  Protein Structure)~\citeappx{klein2019tabular}, and
  \item NMT-Bench (Somali to English, Swahili to English, Tagalog to English)~\citeappx{zhang2020reproducible}.
\end{enumerate}
The search spaces for each benchmark are presented in
Tables~\ref{appendix:details-of-experiments:tab:hpolib},
\ref{appendix:details-of-experiments:tab:nmt-bench}.

HPOlib is an HPO tabular benchmark for neural networks on regression tasks.
This benchmark has four regression tasks
and provides the number of parameters, runtime,
and training and validation mean squared error (MSE) for each configuration.
In the experiments, we optimized validation MSE and runtime.

NMT-Bench is an HPO tabular benchmark for neural machine translation (NMT)
and we optimize hyperparameters of transformer.
This benchmark provides three tasks:
NMT for Somali to English, Swahili to English, and Tagalog to English.
Since the original search space of ``Swahili to English'' is slightly different from
the others, we reduced its search space to be identical to the others.
In the experiments, we optimized
translation accuracy (BLEU)~\citeappx{papineni2002bleu}
and decoding speed.
Note that as mentioned in Footnote 7 in the original paper~\citeappx{zhang2020reproducible},
some hyperparameters do not have both BLEU and decoding speed
due to training failures;
therefore, we pad the worst possible values for such cases.

\subsection{Details of Optimization Methods}

\subsubsection{Baseline Methods}
In the experiments, we used the following:
\begin{enumerate}
  \item MO-TPE~\citeappx{ozaki2022multiobjective},
  \item Random search~\citeappx{bergstra2012random},
  \item TST-R~\citeappx{wistuba2016two}, and
  \item RGPE~\citeappx{feurer2018practical}.
\end{enumerate}
Note that since TST-R and RGPE are introduced
as meta-learning methods for single-objective optimization problems,
we extended them to MO settings using either ParEGO~\citeappx{knowles2006parego}
or EHVI~\citeappx{emmerich2011hypervolume}.
Since both extensions require only ranking loss between configurations,
we defined the ranking of each configuration using
the non-dominated rank and crowding distance
as in NSGA-II~\citeappx{deb2002fast}.
Although TST-R and RGPE implementations were provided in \url{https://github.com/automl/transfer-hpo-framework},
since the implementations are based on SMAC3~\footnote{
  SMAC3: \url{https://github.com/automl/SMAC3}
}, which does not support EHVI,
we re-implemented RGPE and TST-R using BoTorch~\footnote{
  BoTorch: \url{https://github.com/pytorch/botorch}
}.
Our BoTorch implementations are also available in \url{https://github.com/nabenabe0928/meta-learn-tpe}.
We chose $\gamma = 0.1$ for MO-TPE by following the original paper~\citeappx{ozaki2020multiobjective}
and modified the tie-break method from HSSP to crowding distance
as TST-R and RGPE also use crowding distance for tie-breaking in our experiments.
Furthermore, we employed multivariate kernel to improve the performance of TPE~\citeappx{falkner2018bohb}.

\begin{figure*}[t]
  \centering
  \includegraphics[width=0.98\textwidth]{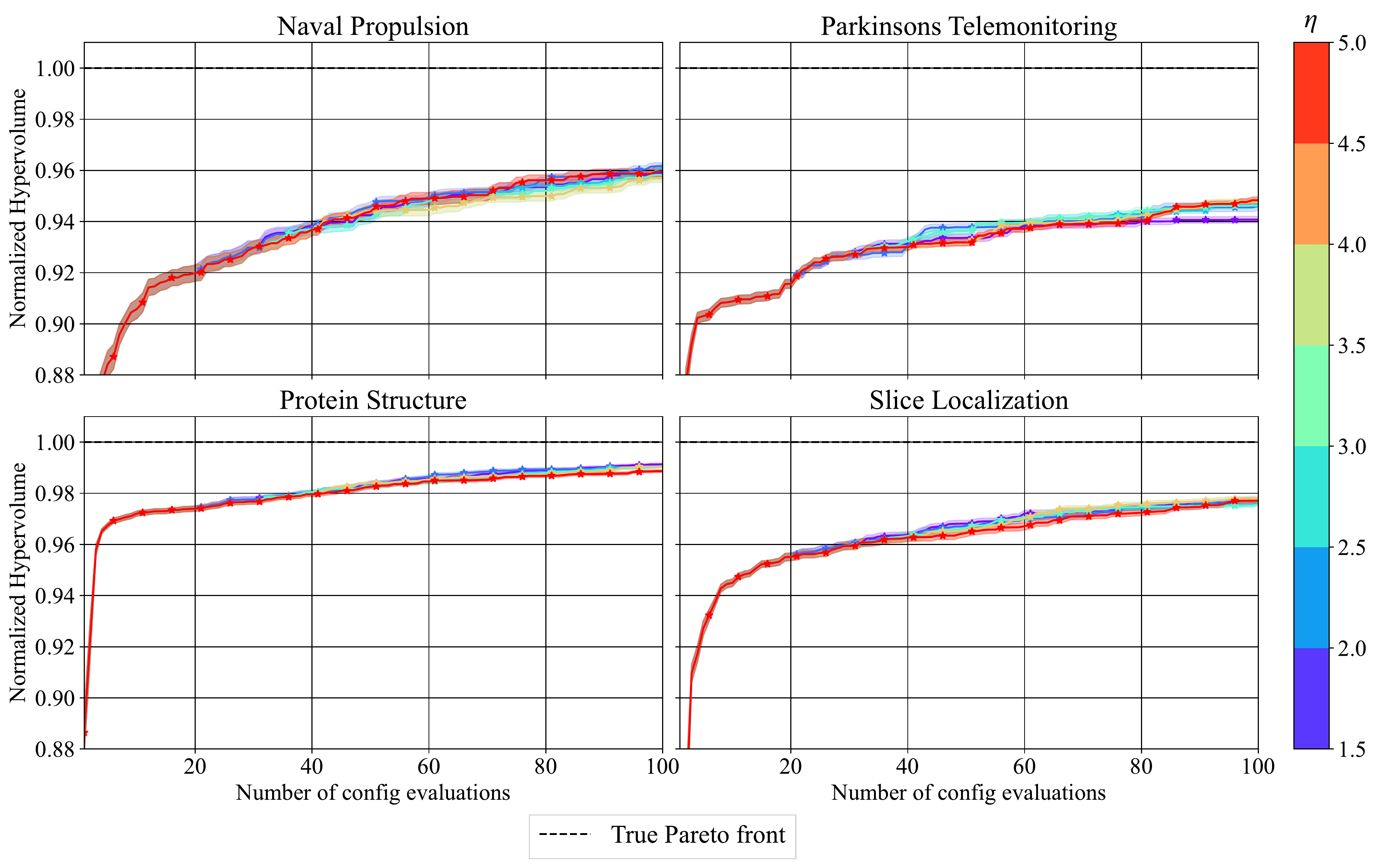}
  \caption{
    The ablation study of $\eta$ in our method by the normalized HV over time on four
    joint neural architecture search and hyperparameter optimization benchmarks (HPOlib)
    from HPOBench.
    Each method was run with $20$ different random seeds
    and the weak-color bands present standard error.
  }
  \label{appx:experiments:fig:hv-hpolib-ablation}
\end{figure*}

\begin{figure*}[t]
  \centering
  \includegraphics[width=0.98\textwidth]{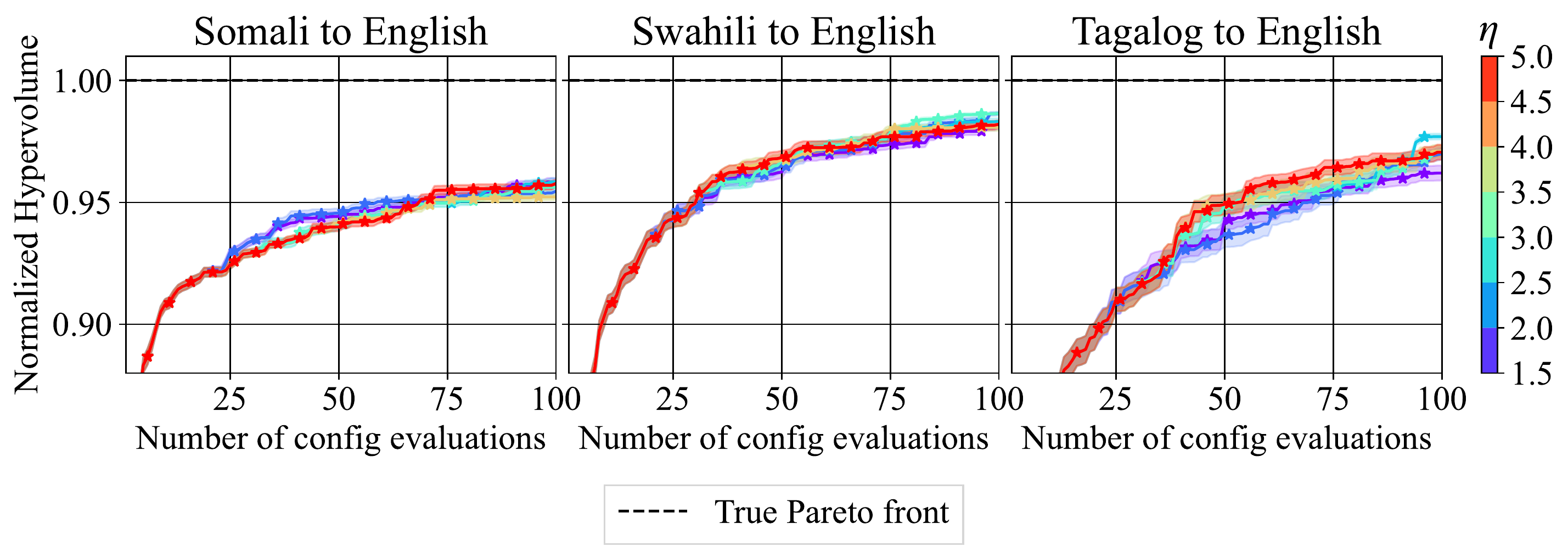}
  \caption{
    The ablation study of $\eta$ in our method by the normalized HV over time on NMT-Bench.
    Each method was run with $20$ different random seeds
    and the weak-color bands present standard error.
  }
  \label{appx:experiments:fig:hv-nmt-ablation}
\end{figure*}

\subsubsection{Settings of Meta-Learning TPE}
As seen in Algorithm~\ref{main:methods:alg:meta-tpe},
meta-learning TPE has several control parameters:
\begin{enumerate}
  \item (TPE) the number of initial samples $N_{\mathrm{init}}$,
  \item (TPE) the number of candidates for each iteration $N_s$,
  \item (TPE) the quantile for the observation split $\gamma$,
  \item (New) dimension reduction factor $\eta$,
  \item (New) sample size of Monte-Carlo sampling for the task similarity approximation $S$, and
  \item (New) the ratio of random sampling in the $\epsilon$-greedy algorithm $\epsilon$.
\end{enumerate}
Note that (TPE) means the parameters already existed from the original version
and (New) means the parameters added in our proposition.
Since the original TPE paper~\citeappx{bergstra2011algorithms,bergstra2013making}
uses $5\%$ of total evaluations for the initialization
and $100$ candidates for each iteration,
we set $N_{\mathrm{init}} = 5$
and $\epsilon = 0.05$ in our experiments
and used $N_s = 100$.
$\gamma = 0.1$ followed the original MO-TPE~\citeappx{ozaki2020multiobjective}.
For the other parameters, we used $S = 1000$ and $\eta = 5/2$.
Note that when $S = 1000$, it does not take a second to compute the similarity even with $N_m = 10^4$
due to the time complexity of $O(S\max_{m \in \{1,\dots,T\}} N_m)$.
The ablation study of $\eta$ is available in the next section.
Furthermore, we employed multivariate kernel as in
prior works~\citeappx{falkner2018bohb,watanabe2023ctpe}
and used crowding distance to tie-break non-dominated rank
as mentioned in the previous section.

\subsubsection{Ablation Study of $\eta$}
Figures~\ref{appx:experiments:fig:hv-hpolib-ablation} and \ref{appx:experiments:fig:hv-nmt-ablation}
show the ablation study of our method.
Intuitively, when $\eta$ is low, more dimensions will be used and the similarity decay will be quicker, and thus low $\eta$ is supposed to be disadvantageous for similar task transfer.
However, our method exhibits almost indistinguishable performance except for \texttt{Parkinsons Telemonitoring} with $\eta=1.5$.
It implies that our method is robust to the choice of the control parameter $\eta$.

\begin{figure*}
  \centering
  \includegraphics[width=0.98\textwidth]{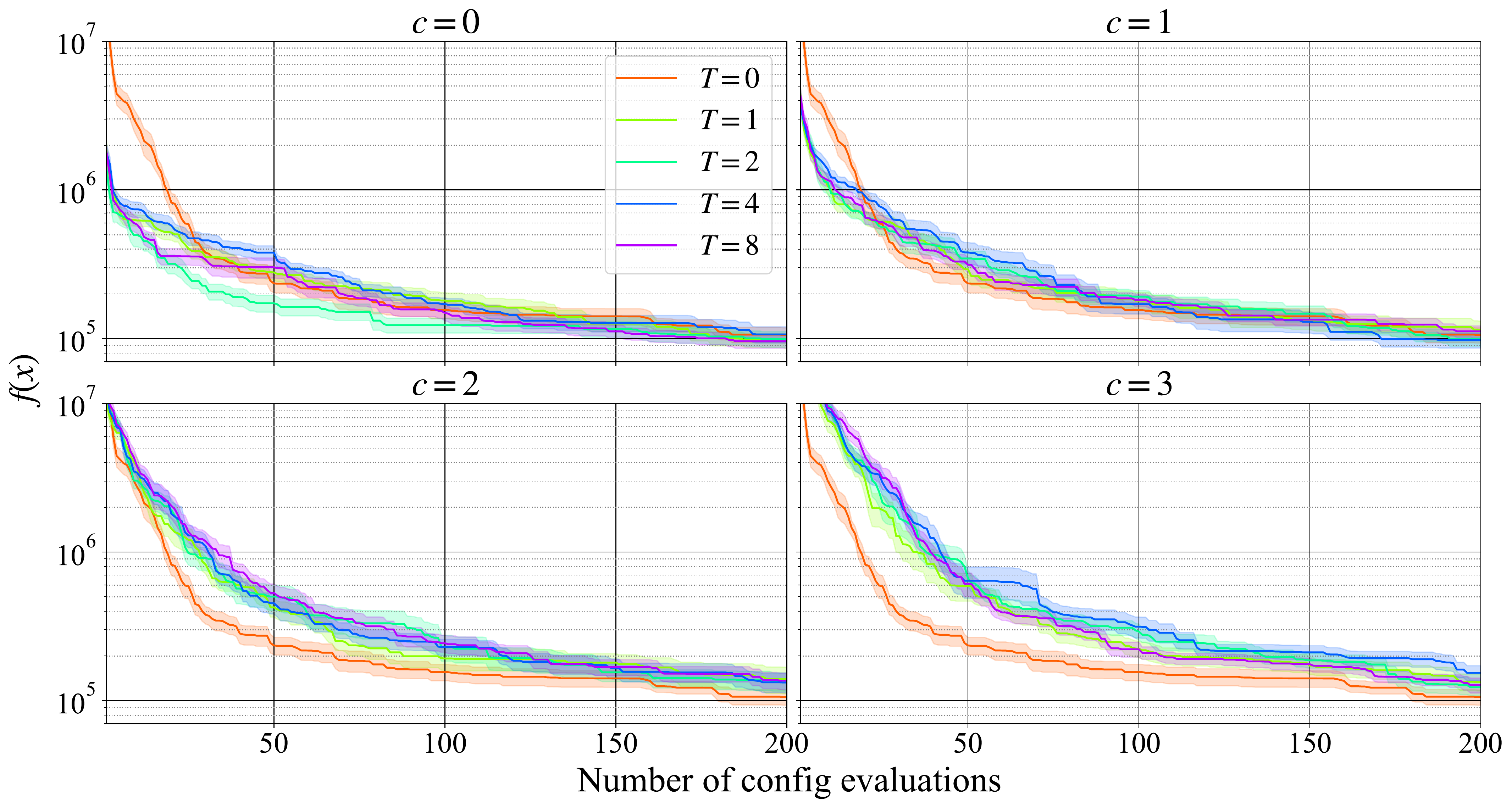}
  \caption{
    The effect of the number of tasks (\textbf{Top left}: $c=0$, \textbf{Top right}: $c=1$, \textbf{Bottom left}: $c=2$, \textbf{Bottom right}: $c=3$. A smaller $c$ implies that the target task and the meta-tasks are more similar.) on the performance of meta-learning TPE.
    $T$ is the number of meta-tasks to use and $T = 0$ is the vanilla TPE.
    Each setting was run with 20 different random seeds and the weak-color bands present standard error.
  }
  \label{appx:experiments:fig:ntasks-analysis}
\end{figure*}

\subsubsection{Effect of Number of Tasks}
To see the effect of the number of tasks, we tested our meta-learning TPE with different numbers of meta-tasks on different similarities of the target function.
In this experiment, we used the $10D$ ellipsoid function $f(\xv | c) =\sum_{d=1}^{10} 5^{d - 1}(x_d - c)^2$ defined on $[-5,5]^{10}$ and used $f(\xv | c = 0)$ as the target task.
Since the contributions that meta-learning can make depend on the meta-tasks to use, we fixed a meta-task and used the same meta-task for all the meta-tasks.
For example, when we use $4$ meta-tasks, we first fixed $c$ to a certain value, e.g. $c=1$, and then we used $c = 1$ for all four tasks.
In this way, we can avoid the similarity difference between different numbers of meta-tasks.
We collected $100$ random observations from each meta-task and used exactly the same setting as in Section~\ref*{main:experiments:section}.
Note that the number of initial configurations $N_{\mathrm{init}} = 10$ is $5\%$ of the maximum number of observations.

Figure~\ref{appx:experiments:fig:ntasks-analysis} shows the results of $c \in \{0,1,2,3\}$.
For all the settings, the performance did not change largely as long as we use the meta-learning.
Although the performance of meta-learning TPE almost recovered that of the vanilla TPE as the number of observations increased, meta-learning TPE was outperformed by the vanilla TPE in the early stage of the optimizations when meta-tasks are dissimilar to the target task.
Note that since we used $D = 10$ instead of $D = 4$, the true task similarity is analytically smaller even with the same $c$ and it explains why our method works poorly for $c = 1, 2$ compared to $D = 4$ provided in Figure~\ref{appx:case-I:fig:similarity-vs-convergence}.

\section{Additional Results of Experiments}
\label{appx:additional-results:section}
Figures~\ref{appx:additional-results:fig:eaf-hpolib}
and \ref{appx:additional-results:fig:eaf-nmt}
show the $50\%$ empirical attainment surfaces~\citeappx{fonseca1996performance,watanabe2023pareto}
for each method on each benchmark.
The $50\%$ empirical attainment surface represents
the Pareto front of the observations
achieved by $50\%$ of independent runs and the implementation is available at \url{https://github.com/nabenabe0928/empirical-attainment-func}.
As discussed in Section~\ref{main:experiments:section},
we mentioned that
our method did not yield the best HV on \texttt{Somali to English}
and \texttt{Protein Structure}.
According to the figures,
validation MSE in \texttt{Protein Structure}
and BLEU in \texttt{Somali to English}
are underexplored by our method.

In Figure~\ref{appx:additional-results:fig:dataset-dist},
we plot the cumulative count
of top configurations by configuration indices
for each dataset
and we can see that both
\texttt{Protein Structure} and
\texttt{Somali to English}
have different distributions, which may imply
the $\gamma$-set for each dataset is not close
to that for the other datasets.
In fact, we could find out 
that only \texttt{Somali to English} requires
a high number of BPE symbols compared to other datasets.
From the results, we could conclude that
while our method is somewhat affected by knowledge transfer
from not similar meta-tasks,
our method, at least, exhibits indistinguishable performance
from that of MO-TPE in our settings.

\bibliographystyleappx{bib-style}
\bibliographyappx{ref}

\fi

\end{document}